\documentclass[10pt,journal,compsoc]{IEEEtran}
\usepackage{algorithm}
\usepackage{algpseudocode}

\usepackage{amsmath, amssymb, amsthm}
\usepackage{graphicx}

\newtheorem{theorem}{Theorem}
\newtheorem{definition} {Definition}

\DeclareMathOperator*{\argmin}{arg\,min}
\DeclareMathOperator*{\argmax}{arg\,max}
\DeclareMathOperator{\tr}{tr}

\newcommand{\vect}[1]{\mathbf{#1}}
\newcommand{\machine}[1]{$\mathcal{M}_{#1}$}
\newcommand{\real}{\mathbb{R}}
\newcommand{\order}{$\mathcal{O}$}
\newcommand{\norm}{\mathcal{N}}
\newcommand{\half}{\frac{1}{2}}
\newcommand{\set}[1]{\mathcal{#1}}
\newcommand{\prob}{\mathbb{P}}
\newcommand{\expect}{\mathbb{E}}
\newcommand{\trace}[1]{\tr\left(#1\right)}

\ifCLASSOPTIONcompsoc
  \usepackage[nocompress]{cite}
\else
  \usepackage{cite}
\fi
\ifCLASSOPTIONcompsoc
  \usepackage[caption=false,font=footnotesize,labelfont=sf,textfont=sf]{subfig}
\else
  \usepackage[caption=false,font=footnotesize]{subfig}
\fi

\begin{document}
%
\title{Learning of Gaussian Processes in Distributed and Communication Limited Systems}
%
%
%
%

\author{Mostafa~Tavassolipour,~Seyed~Abolfazl~Motahari, Mohammad~Taghi~Manzuri~Shalmani
\IEEEcompsocitemizethanks{\IEEEcompsocthanksitem M. Tavassolipour, S. A. Motahari, and M. T. Manzuri Shalmani are with the Department
of Computer Engineering, Sharif University of Technology, Tehran, Iran.\protect\\
}
}

%
%

\markboth{IEEE Transaction on Pattern Analysis and Machine Intelligence}
{Shell \MakeLowercase{\textit{et al.}}: Bare Demo of IEEEtran.cls for Computer Society Journals}
%



\IEEEtitleabstractindextext{%
\begin{abstract}
It is of fundamental importance to find algorithms obtaining optimal performance for learning of statistical models in distributed and communication limited systems. Aiming at characterizing the optimal strategies, we consider learning of Gaussian Processes (GPs) in distributed systems as a pivotal example. We first address a very basic problem: how many bits are required to estimate the inner-products of Gaussian vectors across distributed machines? Using information theoretic bounds, we obtain an optimal solution for the problem which is based on vector quantization. Two suboptimal and more practical schemes are also presented as substitute for the vector quantization scheme. In particular, it is shown that the performance of one of the practical schemes which is called per-symbol quantization is very close to the optimal one. Schemes provided for the inner-product calculations are incorporated into our proposed distributed learning methods for GPs. Experimental results show that with spending few bits per symbol in our communication scheme, our proposed methods outperform previous zero rate distributed GP learning schemes such as Bayesian Committee Model (BCM) and Product of experts (PoE).
\end{abstract}

}

\maketitle

\IEEEdisplaynontitleabstractindextext

%
\IEEEpeerreviewmaketitle

\IEEEraisesectionheading{\section{Introduction}\label{sec:introduction}}

%
%
%
%
\IEEEPARstart{C}{omplex} systems are ubiquitous; biological, social, transportation networks are examples that we face them on a daily basis. Managing and controlling such systems require learning from data observed and measured at possibly distinct points of the systems. Extracted Data have several intrinsic features: they contain various forms of attributes which are noisy and incomplete; they are massive and usually generated distributively across the networks. Effective and efficient utilization of available resources plays an important role in learning, managing, and controlling of such complex systems. Through  a distributed and joint communication-computation design,  this paper aims at presenting important machine learning instances where improvement in  performance compared to conventional methods can be analytically and experimentally justified. 

Many frameworks are developed to handle complex and distributed datasets efficiently. Hadoop Distributed File Systems (HDFS) \cite{shvachko2010hadoop}, Google File System (GFS) \cite{ghemawat2003google} are examples of distributed storage systems. MapReduce \cite{dean2008mapreduce} and Spark \cite{zaharia2010spark} are well known frameworks for distributed data processing systems. These platforms can be used to implement machine learning algorithms naturally \cite{meng2016mllib}.

Distributed processing and storage systems are designed to be universal by separating  computation and communication tasks. The aim of communication is to exchange information between distributed machines efficiently, whereas the aim of computation is to run a distributed algorithm with the given communication platform. However, communication can be adjusted and modified based on the problem at hand achieving superior results. A new trend in research is started to design  distributed and joint communication-computation optimal systems \cite{jordan2015machine}. 




In this paper, we present some important machine learning instances in distributed settings and determine analytically and experimentally how one can obtain the optimal trade-off between performance and communication cost. The problem that we focus on  is learning of \emph{Gaussian Processes} (GPs) which are fundamental models applicable to many machine learning tasks. 
We propose efficient methods for GP learning and analyze the performance and communication cost of the proposed methods. Although GPs can be used for both classification and regression, we only consider regression in presented applications and experiments.  However, the proposed methods are also applicable to other learning models such as distributed manifold learning, KNN-based algorithms, etc.

There exist abundant studies related to distributed learning. Here, we address a few of them which we have found closest to the setting considered in this paper. From information theoretic point of view, distributed statistical inference is addressed by Ahlswede \cite{ahlswede1986hypothesis}, Amari and Han \cite{amari1998statistical, amari1995parameter, han1987hypothesis}. In their settings, data are assumed to be distributed between two machines and through minimal communication, it is desired to test a hypothesis or estimate a parameter optimally.  They used the \emph{method of types} \cite{cover2012elements}   to obtain  some useful error and communication bounds on several problems. A survey by Han et. al. \cite{amari1998statistical} summarizes all early studies in that direction where \emph{hypothesis testing}, \emph{parameter estimation}, and \emph{classification} problems in distributed settings are addressed.

Recently, several researchers worked on distributed statistical learning problems and used tools from information theory to obtain bounds on the performance of the algorithms: \cite{rahman2012optimality,jordan2016communication, duchi2014optimality, zhang2013information}. For instance, Duchi et. al. \cite{duchi2014optimality} investigated the minimum amount of communication required for achieving centralized minimax risk in certain estimators. Rahman et. al. \cite{rahman2012optimality} analyzed the problem of distributed hypothesis testing from information theoretic perspective and prove the optimality of binning method for conditionally independent sources.

There are papers addressing the problem of distributed statistical inference from machine learning point of view  \cite{liu2014distributed, meng2013distributed, broderick2013streaming, duchi2014optimality, balcan2015communication}. For instance, Liu in \cite{liu2014distributed} showed that how the maximum likelihood estimation of parameters of exponential families can be obtained distributively. Meng et. al. \cite{meng2013distributed} proposed a distributed algorithm for learning of inverse covariance matrix in Gaussian graphical models. Broderick et. al. \cite{broderick2013streaming} proposed a distributed algorithm for calculating the posterior distribution using variational inference over distributed datasets. Authors of \cite{balcan2015communication} have addressed the kernel PCA in distributed systems on a large amount of data.

Distributed machine learning problems are also related to  distributed optimization problems as optimization of  an objective function during their training or test procedure is part of learning tasks.  Therefore, several studies on  distributed optimization algorithms \cite{boyd2011distributed} can directly or indirectly be applied to  many machine learning models. Alternating Direction Method of Multipliers (ADMM) \cite{boyd2011distributed} is one of the most successful distributed algorithms for solving optimization problems over distributed datasets. There are many papers that are based on the ADMM algorithm \cite{zhang2014asynchronous, forero2010consensus, liu2012distributed}. For example, Forero et. al. \cite{forero2010consensus} proposed a re-parameterization for SVM learning in distributed manner based on ADMM. 

As mentioned before, in this paper we have considered GP learning for distributed systems. There are many works that propose methods for learning GPs distributively \cite{deisenroth2015distributed}. Deisenroth et. al. \cite{deisenroth2015distributed} presented a distributed model for GP learning which is based on models called  \emph{product of experts} (PoE). PoE-based models assume that the local dataset in each machine is independent of the other local datasets. Thus, the marginal likelihood factorizes into product of individual local marginal likelihoods.
%
In these models, the overall gram matrix will be block diagonal where the blocks are the local gram matrices of the machines. In PoE models, after training the model hyper-parameters, predictions of GP experts are combined. There are several combining methods such as PoE \cite{ng2014hierarchical}, Bayesian Committee Machine (BCM) \cite{tresp2000bayesian} and the recently proposed method robust BCM (rBCM) \cite{deisenroth2015distributed}.

PoE models need no data transmission during training which makes them favorable when communication is very restricted. However, assuming independence of local experts is unfavorable in many applications. For example, in RBF kernels that are non-increasing function of Euclidean distances of inputs, such assumption may decrease the performance of PoE models. This is due to the fact that data is distributed randomly between machines and nearby inputs may be located at distinct machines. In this case, transmission of inputs is necessary to achieve the optimal performance of the learning task.

There are other notable studies on localized and distributed learning of sparse GPs such as \cite{titsias2009variational, quinonero2005unifying, gal2014distributed, hensman2013gaussian, rasmussen2006gaussian}. The main idea behind  these works is to select or extract $m \ll n$ inducing points from the training datasets so that the trained sparse GP is  close  to the full GP.  Authors of \cite{gal2014distributed} proposed a re-parameterization on \cite{titsias2009variational} and \cite{titsias2010bayesian} for applying sparse GPs on a network of multiple machines. This method for obtaining inducing points, in each iteration of the optimization procedure transmits all inducing and hyperparameters of the GPs to all machines. Due to the size of the inducing set and the number of iterations till the convergence, the communication cost of this method is usually very high. Moreover, the exact communication cost of this method is not calculable before running the training procedure.

This paper is organized as follows. Section \ref{sec:gp background} briefly describes background on GPs for the regression problem. In Section \ref{sec:problem definition}, we formally define the main problem of the paper. In Section \ref{sec:trans schemes}, we propose three methods for solving the inner-product estimation problem in a distributed system. In section \ref{sec:distr gp}, we show how the proposed methods of Section \ref{sec:trans schemes} can be used for distributed learning of GPs. Finally, in Section \ref{sec:experiments}, we evaluate our methods over synthetic and real datasets.

\section{Gaussian Processes} \label{sec:gp background}
Statistical machine learning models can be divided into \emph{parametric} and \emph{non-parametric} models. GPs are powerful non-parametric models that are used in many supervised learning tasks such as classification and regression \cite{rasmussen2006gaussian}.
Basically, a GP is a stochastic process that any finite collection of its points is jointly Gaussian. To learn a GP model, it is sufficient to estimate its mean and covariance functions. The covariance matrix of a GP is sometimes called \emph{gram matrix} which is denoted by $G$ through out this paper.
 

In regression problems, we wish to estimate a real-valued function $f(x)$ where $x \in \real^d$. We can use a GP as a prior distribution over these real-valued functions. Suppose our training dataset is $\{(x_i, y_i)\}_{i=1}^n$, which we call $\{x_i\}$ input samples and $\{y_i\}$ target observations. The training examples are noisy realization of the target function $f(x)$ where each $y_i$ is obtained by adding a Gaussian noise to $f(x)$ at input $x_i$ as follows:
\begin{equation}
	y_i = f(x_i) + \epsilon_i, \quad \epsilon_i \sim \norm(0, \sigma^2_\epsilon).
\end{equation}
The posterior process is a GP with the following mean and covariance functions: 
\begin{align}
	m(x) &= G_{x n} (G + \sigma^2_\epsilon I)^{-1} \vect{y} \\
	cov(x, x') &= k(x, x') + G_{x n} (G + \sigma^2_\epsilon I)^{-1} G_{n x'}, 
\end{align}
where $G_{x n} = [k(x, x_1), \cdots, k(x, x_n)]$ is an $n$-dimensional row vector of kernel values between $x$ and the training inputs, $\vect{y}$ is the vector of all target observations $\vect{y} = [y_1, \cdots, y_n]^T$, and $G_{n x} = G_{x n}^T$. Then, the predictive distribution for a new input $x_*$ is $\prob(y_* | x_*, \vect{y}) = \norm\left(m(x_*), cov(x_*, x_*) + \sigma_\epsilon^2 \right)$. Note that $G$ is an $n \times n$ matrix where  $G_{ij}$ is the correlation of $f(x_i)$ and $f(x_j)$.


In the training procedure, the hyperparameters of the model (e.g. $\sigma_\epsilon^2$, and parameters of the kernel function) need to be trained. Usually, the hyperparameters are obtained by maximizing the marginal likelihood $\prob(\vect{y})$.
Although there exist many kernel functions such as linear, polynomial, squared exponential, Matern, and etc \cite{rasmussen2006gaussian},
we select the following linear kernel function:
\begin{equation}\label{eq:linear-kernel}
	k(x_i, x_j) = a~x_i^T x_j+b,
\end{equation}
where $a$ and $b$ are the hyperparameters of the kernel.
Although our proposed methods for GP learning is based on the linear kernel, they can be readily applied on any Radial Basis Function (RBF) kernels as experimentally  will showed in Section \ref{sec:experiments}. It is worth mentioning that kernels that are only a function of $\Vert x_i - x_j \Vert$ are called RBFs.

\section{Problem Definition} \label{sec:problem definition}
Suppose we have only two machines in our distributed system so that each machine stores some portion of the whole dataset. Then the gram matrix can be partitioned as follows:
\begin{equation}
G =
\begin{bmatrix}
G_{11} &	G_{12}	\\
G_{21} &	G_{22}
\end{bmatrix},
\end{equation} 
where $G_{11}$ (resp. $G_{22}$) is the gram matrix of data stored in the first (resp. second) machine and $G_{12}$ is the cross gram matrix between data of the first and second machines. Moreover, $G_{12}= G_{21}^T$. Note that $G_{11}$ and $G_{22}$ can be calculated locally without any data transfer between machines. However, $G_{12}$ and its transpose $G_{21}$ must be calculated cooperatively by the two machines which requires intensive communication.  

Let the first machine be responsible for learning the GP model. This machine calculates $G_{11}$ locally, whereas $G_{12}$ and $G_{22}$ need to be calculated using messages received from the second machine. In Section \ref{sec:distr gp}, we will show that if $G_{12}$ is available, then $G_{22}$ can be efficiently estimated. Thus, our main goal is to find an efficient method to obtain $G_{12}$ in distributed manner with minimum distortion.

In fact, we seek for an algorithm that has low communication cost and at the same time has low distortion in reconstruction of the target matrix. Without any limitation on communication cost, we gather the whole dataset at  the first machine which leads to the best performance with zero distortion. In Section \ref{sec:trans schemes}, we propose novel methods respecting imposed communication constraints and achieving optimal/suboptimal performance in the learning procedure. 

In our basic setting, the gram matrix of GP is first estimated by communicating information bits between machines. The estimated gram matrix is then used to learn the latent function. Therefore, minimum distortion in estimating of the gram matrix is desired. By considering a linear kernel function for the GP model, our goal is to calculate the inner product of two datasets that stored in different machines with minimum distortion. 

To put it formally, consider two datasets $X = \{x_1, \cdots, x_n\}$ and $Y = \{y_1, \cdots, y_n\}$ stored in machines \machine{x} and \machine{y}, respectively. Samples $x_i$ and $y_i$ are $d$-dimensional random vectors with zero means. We would like to calculate the inner product $\langle x_i, y_j \rangle$ for any $x_i \in X$ and $y_j \in Y$ in machine \machine{y}. Exact calculation of inner products requires transmission of the whole dataset $X$ to machine \machine{y} which is not efficient in many practical situations due to its communication cost. Lowering the communication cost, however, results in reduction in the quality of computed inner products. Therefore, finding the tradeoff between the communication cost and the quality of calculated inner products plays an important role in the design of practical systems. 

To characterize the tradeoff, we limit the number of transmission bits per observation to $R$ bits; leading to the total of $nR$ transmission bits. These bits are used to design an encoding map $\phi$ from X to an index set $\mathcal{I} = \left\{ 1,2, \cdots, 2^{nR} \right\}$. The encoded index is transmitted to machine \machine{y} where a decoding map $\psi$ is used to map the index to another set $\hat X = \{\hat x_1, \cdots, \hat x_n \}$ which is called the set of reconstructed points. In other words, each $x_i$ is quantized to $\hat{x}_i$.

We would like to design $\phi$ and $\psi$ such that the distortion in computing the inner products is minimized. The distortion is  measured by
\begin{align}\label{eq:err-func-org}
D &= \frac{1}{n^2} \sum_{i=1}^{n} \sum_{j=1}^{n} \Big(\langle x_i,  y_j \rangle - \langle \hat{x}_i, y_j \rangle \Big) ^ 2,
\end{align}
that is,
\begin{equation}\label{eq:err-func}
D = \frac{1}{n} \sum_{i=1}^{n} (x_i-\hat{x}_i)^T S_y (x_i-\hat{x}_i),
\end{equation}
where $S_y$ is the sample covariance matrix of $Y$.
\begin{definition}
Achievable $(R, D)$\\
A rate-distortion pair $(R,D)$ is achievable if there exist encoding and decoding maps $\phi$ and $\psi$ such that the transmission rate is limited by $R$ and the distortion in reconstructing the inner products is limited by $D$.
\end{definition}
In this paper, we are interested in characterizing the set of all achievable rate-distortion pairs. 
The region of all achievable rate-distortion pairs is a convex region. To show this, one needs to prove that if $(R_1, D_1)$ and $(R_2, D_2)$ are two achievable pairs, then $(\lambda R_1 + (1 - \lambda) R_2, \lambda D_1 + (1-\lambda)D_2)$ is an achievable pair for any $\lambda\in[0,1]$. A coding scheme which uses both schemes for different portions of data can achieve this. Let $\lambda$ portion of data be encoded by the first scheme and the rest with the second scheme. It is easy to show that the overall rate-distortion pair $(\lambda R_1 + (1 - \lambda) R_2, \lambda D_1 + (1-\lambda)D_2)$ is indeed achievable.

Extreme points of the achievable region can be obtained by solving the following optimization problem
\begin{align}
&\min_{ (\phi ,  \psi)}  \mu R + (1-\mu) D,	\quad \mu \in [0, 1].
\end{align}

\section{Transmission Schemes} \label{sec:trans schemes}
In this section, we present three different schemes for transmitting the dataset $X$ to machine \machine{y} with low communication cost and distortion. The first scheme is based on an optimal compression method which characterizes rate-distortion region theoretically for normally distributed samples. Although the scheme is optimal, it is not practical due to computation complexity required for block coding. Thus, we propose another scheme that is an approximate version of the optimal one. This method uses per symbol coding instead of block coding. Experiments show its performance is near optimal and it is applicable for many distributed settings. The third scheme is even more simpler and does not require any nonlinear map. It is based on ideas from dimension reduction schemes which are tailored to our distributed setting. This method has more communication cost than the previous two schemes, but it is very simple and fairly efficient in many applications.

\subsection{Compression}

Let us assume the datasets $X$ and $Y$ are drawn from two independent distributions as
\begin{align}
\mathbb{P}(X,Y) & = \mathbb{P}(X)\mathbb{P}(Y)\nonumber\\
& = \prod_{i=1}^n \mathbb{P}(x_i) \prod_{i=1}^n \mathbb{P}(y_i).
\end{align}
Taking the expectation of distortion in \eqref{eq:err-func} with respect to $y$, we obtain
\begin{align} \label{eq:expected distortion-func}
\expect_{\prob(y)}\{D\} &=\frac{1}{n} \sum_{i=1}^{n}(x_i - \hat{x}_i)^T Q_y (x_i - \hat{x}_i),
\end{align}
where $Q_y$ is the covariance matrix of $Y$ defined by 
$Q_y = \mathbb{E} \{ y^T y\}.$ It is worth mentioning that under the preceding conditions, the covariance matrix of $y$ is the only parameter influencing the achievable rate-distortion region and therefore the knowledge of the exact distribution of $Y$ is not required. 

We would like to characterize rate-distortion region for the above setting. This problem is an instance of the rate-distortion problem in information theory where the distortion is defined by \eqref{eq:err-func}. The optimal coding scheme achieving points on the boundary of the achievable region is obtained by solving the following optimization problem \cite{cover2012elements}:
\begin{align}\label{eq:rate-distortion}
R &= \min_{p(\hat{x} \mid x)} I(x; \hat{x})		\\
&\text{subject to:}	\nonumber \\
&\qquad \mathbb{E}_{P(x,\hat{x})}\{d(x, \hat{x})\} \leq D, \nonumber	\label{eq:compression-constraint}
\end{align}
where
\begin{align}
d(x, \hat{x}) &= (x - \hat{x})^T Q_y (x - \hat{x}).
\end{align}
Suppose $P^*(\hat x | x)$ is the solution of \eqref{eq:rate-distortion}. The coding scheme with the rate $R$ and distortion $D$ is obtained as follows:
\begin{enumerate}
\item [] \textbf{\emph{Generating Codebook:}}
Calculate $P^*(\hat x) = \int P(x) P^*(\hat x | x) dx$. Generate a codebook consisting of $2^{n R}$ sets of size $n$ that are drawn i.i.d. from $\prod_{i} P^*(\hat{x}_i)$. We denote the generated random sets by $\hat{X}_1, \cdots, \hat{X}_{2^{n R}}$. The codebook is available at both transmitter (\machine{x}) and receiver (\machine{x}) machines. Note that the codebook can be generated at both machines by an equal initial seed which leads to identical codebooks. Thus, there is no need to send the codebook from a machine to the other.
\item [] \textbf{\emph{Encoding:}} We are given a set $X$. Index it by $j$ if there exists a set in the codebook such that the pair $(X, \hat{X}_j)$ is strongly jointly typical (for definition of strong typicality refer to \cite{cover2012elements}). If there is more than one such $j$, send the first in lexicographic order. If there is no such $j$, send 1 to the receiver. It is clear that transmitting the index $j$ requires $n R$ bits. Since by each transmission we send $n$ data samples, the average rate will be $R$ bits per each sample.
\item [] \textbf{\emph{Decoding:}} At the receiver (machine \machine{y} in our case), select $\hat{X}_j$ as the reproduced set for $X$.
\end{enumerate}
For the proof that this coding scheme asymptotically (when $n \to \infty$)  has rate $R$ and distortion $D$ see \cite{cover2012elements}.
Thus, by solving the optimization problem \eqref{eq:rate-distortion} and using the above coding scheme, we obtain the optimal method for transmitting the set $X$ with rate $R$ and distortion $D$. Unfortunately, in many cases, there is no closed form solution for \eqref{eq:rate-distortion}, but we can obtain a lower bound for it (Theorem \ref{thm:compression}). 

The following theorem presents a lower bound for rate-distortion function \eqref{eq:rate-distortion}.

\begin{theorem} \label{thm:compression}
Let $\{x_1, \cdots, x_n\}$ be $n$ $i.i.d.$ $d$-dimensional zero mean random vectors with covariance matrix $Q_x$, and let the distortion measure be $d(X, \hat{X}) = \sum_{i=1}^{n}(x_i - \hat{x}_i)^T Q_y (x_i - \hat{x}_i)$, where $Q_y$ is a symmetric positive definite matrix. Then the rate distortion function \eqref{eq:rate-distortion} is lower bounded by
\begin{equation}\label{eq:compression_lb}
R_{lb}(D) = h(x) - \frac{1}{2} \log (2 \pi e)^{d} \det(Q_y^{-1}) - \frac{1}{2} \sum_{i=1}^{d} \log q_i,
\end{equation}
where
\begin{equation} \label{eq:q_i}
q_i = 
\left\{
\begin{array}{lr}
\lambda &  \quad if \quad \lambda < \Lambda_{ii} \\
\Lambda_{ii} & \quad if \quad \lambda \geq \Lambda_{ii}
\end{array}
\right.,
\end{equation}
$h(x)$ is entropy of variable $x$, $\Lambda$ is the diagonal matrix of eigenvalues of $Q_x Q_y$, and $\lambda$ is chosen so that $\sum_{i=1}^{d} q_i = D$.
\end{theorem}
\begin{proof}
The lower bound is obtained as follows. Let  $Q = \mathbb{E}\{(x - \hat{x}) (x - \hat{x})^T\}$, we have
\begin{align}
I(x; \hat{x}) &= h(x) - h(x \mid \hat{x})	\nonumber \\
&= h(x) - h(x - \hat{x} \mid \hat{x}) \nonumber\\
&\stackrel{\text{(a)}}{\ge} h(x) - h(x - \hat{x}) \nonumber	\\
&\stackrel{\text{(b)}}{\ge} h(x) - \frac{1}{2} \log (2 \pi e)^d \det(Q) \label{eq:I_lower bound}
\end{align}
where (a) follows from the fact that conditioning does not increase the entropy; (b) follows from the fact that Gaussian distribution maximizes the entropy conditioned on a covariance matrix.   

One can find the best lower bound in \eqref{eq:I_lower bound} by minimizing $\log \det(Q)$ over all appropriate covariance matrices $Q$. This amounts to solving the following optimization problem:
\begin{align} \label{eq:Q optimization prob}
&\min_Q \log \det(Q)	\\
&\text{subject to:}	\nonumber	\\
&\qquad \mathbb{E}_{P(x,\hat{x})}\{d(x, \hat{x})\} = \tr(Q Q_y)\leq D, \nonumber \\
&\qquad Q_x - Q \succeq 0.	 \nonumber
\end{align}
where the second constraint comes from the fact that independence of $u = x- \hat x$ and $\hat x$ results in an equality in (a). More precisely, $x = u + \hat x$ where $u$ and $\hat x$ are independent; this yields $Q_x = Q_u + Q$ and thus, $Q_x \succeq Q$.

It can be shown that the above problem is convex and therefore KKT conditions provide necessary and sufficient conditions for the optimal solution. To obtain the KKT conditions, we first obtain the Lagrangian as
\begin{eqnarray}
\mathcal{L} = \log \det (Q) - \lambda \tr(Q Q_y) - \tr\left(A^T (Q - Q_x)\right),	\nonumber \\
\quad \lambda \geq 0, \quad A \succeq 0
\end{eqnarray}
where $\lambda$ and $A$ are the Lagrange multipliers. Using $\mathcal{L}$, the KKT conditions amount to
\begin{align}
\frac{\partial \mathcal{L}}{\partial Q} = Q^{-1} - \lambda Q_y - A & = 0, \label{eq:Lagrange_Q_der}\\
\tr \left( A^T (Q - Q_x) \right) &= 0,	\label{eq: constraint_Q_less_Q_x}	\\
\tr (Q Q_y) &= D.  \label{eq:Q constraint}
\end{align}
Since $A^T (Q - Q_x)$ is positive semi-definite, \eqref{eq: constraint_Q_less_Q_x} yields
\begin{equation}
A^T (Q - Q_x) = 0 \label{eq:temp_eq 1}.
\end{equation}
By substituting $A$ from \eqref{eq:Lagrange_Q_der} in \eqref{eq:temp_eq 1}, we obtain:
\begin{align}
(Q^{-1} - \lambda Q_y)(Q - Q_x) &= 0 \label{eq:Q_der main-eq}.
\end{align}
Since $Q_y$ is positive definite, it can be decomposed as $Q_y = Q_y^{1/2} Q_y^{1/2}$ where $Q_y^{1/2}$ is the symmetric squared root of $Q_y$.
Multiplying \eqref{eq:Q_der main-eq} by $Q^{-1/2}$ from left, $Q^{1/2}$ from right and $Q_y^{-1/2} Q_y^{1/2}$ from middle, yields
\begin{equation}
Q_y^{-1/2} \left(Q^{-1} - \lambda Q_y\right) Q_y^{-1/2} Q_y^{1/2} \left(Q - Q_x\right) Q_y^{1/2}= 0,
\end{equation}
that is
\begin{multline}
\left(\left(Q_y^{1/2} Q Q_y^{1/2}\right)^{-1} - \lambda I \right) \\
\left(Q_y^{1/2} Q Q_y^{1/2} - Q_y^{1/2} Q_x Q_y^{1/2}\right) = 0  \label{eq:temp 2}.
\end{multline}
We can decompose $Q_y^{1/2} Q_x Q_y^{1/2}$ as follows
\begin{equation} \label{eq: Q_x Q_y svd}
Q_y^{1/2} Q_x Q_y^{1/2} = U \Lambda U^{T},
\end{equation}
where $U$ is a unitary matrix. Since the eigenvalues of $Q_y^{1/2} Q_x Q_y^{1/2}$ are equal to eigenvalues of $Q_x Q_y$, $\Lambda$  is a diagonal matrix which includes eigenvalues of $Q_x Q_y$.
By substituting \eqref{eq: Q_x Q_y svd} in \eqref{eq:temp 2} we can rewrite \eqref{eq:temp 2} as follows
\begin{multline} \label{eq:temp 4}
U^T\left(\left(Q_y^{1/2} Q Q_y^{1/2}\right)^{-1} - \lambda I \right)U U^T \\
\left(Q_y^{1/2} Q Q_y^{1/2} - U \Lambda U^T \right) U = 0.
\end{multline}
Here, \eqref{eq:temp 2} is multiplied by $U^T$ from left, $U$ from right, and $U U^T$ from the middle.
Defining $\tilde{Q} =U^T Q_y^{1/2} Q Q_y^{1/2} U$, from \eqref{eq:temp 4} obtains
\begin{align*}
&(\tilde{Q}^{-1} - \lambda I)(\tilde{Q} - \Lambda) = 0.
\end{align*}
The above equation forces $\tilde{Q}$ to be diagonal. Each element on diagonal of $\tilde{Q}$ is whether $1/\lambda$ or $\Lambda_{ii}$. 
Clearly, since $Q =  Q_y^{-1/2} U \tilde{Q} U^T Q_y^{-1/2}$, any solution for $Q$ can be mapped uniquely to $\tilde{Q}$  and vice versa. It is easy to show that $\tilde{Q}$ which is defined as follows satisfies all the KKT conditions and therefore, after mapping it to $Q$, it is an optimal solution for \eqref{eq:Q optimization prob}
\begin{align}
\tilde{Q} &= diag(q_1, ..., q_d)	 \label{eq: Q tilde}\\
q_i &= 
\left\{
\begin{array}{lr}
1/\lambda &  \quad 1/\lambda < \Lambda_{ii}	\\
\Lambda_{ii} & \quad \text{otherwise}
\end{array}
\right.
\end{align}
where $\lambda$ is chosen so that:
\begin{equation}
\sum_{i=1}^{d} q_i = D.
\end{equation}
Finally, by substituting $Q =  Q_y^{-1/2} U \tilde{Q} U^T Q_y^{-1/2}$ in \eqref{eq:I_lower bound} we obtain the lower bound \eqref{eq:compression_lb}.
\end{proof}
Theorem \ref{thm:compression} gives a lower bound on the rate distortion function \eqref{eq:rate-distortion}. In some cases such as the Gaussian distribution the lower bound is achievable (see Theorem \ref{thm:normal achiveability}), but it is not the case in general. Although this theorem does not propose a coding scheme for data compression, it gives some useful clues for designing low communication rate coding schemes. According to the lower bound $R_{lb}(D)$ in \eqref{eq:compression_lb}, eigenvalues of the product matrix $Q_x Q_y$ plays an important role in bit rate allocation. In other words, dimensions with larger eigenvalues need more bits compared to dimensions with small eigenvalues. This intuition lead us to propose an approximate coding scheme in the next section for Gaussian datasets.

The following theorem shows that the lower bound \eqref{eq:compression_lb} is achievable when the dataset $X$ is drawn from a normal distribution. In this case, the achievable rate distortion region can be fully characterized as the following theorem presents.

\begin{theorem} \label{thm:normal achiveability}
If $\{x_1, \cdots, x_n\}$ are drawn $i.i.d.$ from a normal distribution $\mathcal{N}(0, Q_x)$, then the lower bound \eqref{eq:compression_lb} is achievable if we choose $\{\hat x_i\}$ as follow:
\begin{equation}
x_i = \hat{x}_i + z_i,\quad \hat{x}_i \sim \mathcal{N}(0, Q_x - Q), \quad z_i \sim \mathcal{N}(0, Q)	\label{eq:x_hat for normal}
\end{equation}
where $\hat{x}_i$ and $z_i$ are independent, and $Q$ is chosen so that $\tr(Q Q_y) = D$ and $Q \preceq Q_x$.
\end{theorem}
\begin{proof}
To find the conditional distribution $\prob(\hat x | x)$ that achieves the lower bound \eqref{eq:compression_lb}, it is usually more convenient to use conditional distribution $\prob(x | \hat x)$. If we choose the joint distribution of $x_i$ and $\hat{x}_i$ using \eqref{eq:x_hat for normal}, then
\begin{equation}
I(x; \hat x) = \half \log \det(Q_x) - \half \log \det(Q).
\end{equation}
By setting $Q =  Q_y^{-1/2} U \tilde{Q} U^T Q_y^{-1/2}$, where $\tilde{Q}$ and $U$ is determined by \eqref{eq: Q tilde} and \eqref{eq: Q_x Q_y svd}, respectively; it is easy to show that $Q$ satisfies $\tr(Q Q_y) = D$ and $Q \preceq Q_x$. Then, we have
\begin{equation}
I(x; \hat x) = \half \log \det(Q_x Q_y) - \half \sum_{i=1}^{d} \log q_i,
\end{equation}
which is equal to the lower bound \eqref{eq:compression_lb}. 
\end{proof}

The theorem \ref{thm:normal achiveability} shows that the optimal compression for Gaussian distribution must be performed using the product of covariance matrices $Q_x$ and $Q_y$. We will see that this result is similar to Theorem \ref{thm:dim_red} where the optimal linear dimension reduction is performed by product of the sample covariance matrices $S_x$ and $S_y$. Theorem \ref{thm:normal achiveability} is in fact the well known reverse-water filling problem on eigenvalues of matrix $Q_x Q_y$. In other words, if we decompose $Q_y^{1/2} Q_x Q_y^{1/2}$ as $U \Lambda U^T$ where $\Lambda$ is a diagonal matrix of eigenvalues of $Q_x Q_y$, then we should consume more bits for dimensions with larger corresponding eigenvalue $\Lambda_{ii}$. According to the theorem, for dimensions with eigenvalues less than a threshold $\lambda$ no bits are allocated (i.e. they are not encoded and transmitted), and for dimensions greater than the threshold we consume bits to reduce distortion for them. In fact, $q_i$ in Theorem \ref{thm:normal achiveability} is the distortion value for $i$-th dimension after compression. 


According to Theorem \ref{thm:normal achiveability}, the machine \machine{x} needs $Q_y$ for compressing dataset $X$, thus, if $Q_y$ is not known for the machine \machine{x}, transmitting it has \order($d^2$) bits communication cost. After compression, transmitting the compressed dataset $\hat{X} = \{\hat{x}_1, \cdots, \hat{x}_n\}$ requires \order($n R$) bits. 

Although the compression method proposed by Theorem \ref{thm:normal achiveability} is optimal and simulation results in section \ref{sec:experiments} (see \figurename{} \ref{fig:distortion-cc}) shows that it significantly outperforms the dimension reduction method, it is unfortunately not practical in real world due to exponential computation complexity for encoding. To overcome this issue, in the next section we have proposed an applicable approximate method for compression which has near optimal performance.  


 
\subsection{Per-Symbol Compression Scheme} \label{sec:approx compression}
Although the proposed method for compression in Theorem \ref{thm:normal achiveability} is optimal for normal data samples, it is not practically favorable to implement in many cases as it requires vector quantizers which are designed to work on blocks of data samples. Thus, a per-symbol quantizer which is easier to implement is favorable. 

In this section, we propose a simple yet effective quantizer for normal data samples which operates near optimal in many practical situations. In the proposed method, we first transform samples of the dataset $X$ by $U^T Q_y^{1/2}$, where $Q_y^{1/2}$ is the squared root of $Q_y$, and $U$ is obtained by solving the following singular value decomposition problem:
\begin{equation}
Q_y^{1/2} Q_x Q_y^{1/2} = U \Lambda U^T.
\end{equation}
It is easy to see that the above transformation makes dimensions of samples $\{x_i\}$ uncorrelated. Hence, for Gaussian samples the dimensions are independent. We denote the transformed dataset by $X' = \{x_1', \cdots, x_n'\}$. Thus, we can consider each $x_i'$ as $d$ independent normal random variables. Finally, we quantize each dimension separately. Our coding scheme is as follows:
\begin{enumerate}
	\item [] \textbf{\emph{Encoder:}} Transform dataset $X = \{x_1, \cdots, x_n\}$ to $X' = \{x_1', \cdots, x_n'\}$ using $x_i' = U^T Q_y^{1/2} x_i$. Quantize each dimension of $\{x_i'\}$ samples using a scalar quantizer. Transmit quantized samples denoting by $\hat{X}' = \{\hat{x}_1', \cdots, \hat{x}_n'\}$ to the receiver.
	\item [] \textbf{\emph{Decoder:}} Transform received data samples using
	\begin{equation}
	\hat{x}_i = Q_y^{-1/2} U \hat{x}_i', \qquad i = 1, \cdots, n,
	\end{equation}
	to reconstruct $\hat{x}_i$.  
\end{enumerate}
In fact, we quantize the transformed version of samples in the encoder and apply the inverse transform on received samples at the decoder. The expected distortion of this method is calculated by taking expectation from \eqref{eq:err-func} as follows:
\begin{align}
\expect\{D\} &= \trace{Q_y \expect\left\{(x - \hat x) (x - \hat x)^T\right\}} \nonumber \\
&=\trace{Q_y \expect\left\{Q_y^{-1/2} U (x' - \hat{x}') (x' - \hat{x}')^T U^T Q_y^{-1/2}\right\}} \nonumber	\\
&= \trace{\expect\left\{(x' - \hat{x}') (x' - \hat{x}')^T\right\}} \nonumber	\\
& = \trace{Q'}, \label{eq:dist of per symbol}
\end{align}
where $Q'$ is the covariance matrix of $x' - \hat{x}'$. Since the dimensions of $x'$ are independent for normal distribution and the quantization is done separately on each dimension, $Q'$ is diagonal, hence by denoting $Q' = diag(\sigma_1, \cdots, \sigma_d)$ the expected distortion is $\sum_{i=1}^{d} \sigma_i$. In fact, $\sigma_i$ is the distortion of the $i$-th dimension. If we want to quantize each sample by $R$ bits, we should allot more bits to the dimensions with larger distortions. 
More precisely, for finding the optimal bit allocation to each dimension, the following integer optimization problem should be solved:
\begin{equation}\label{eq: bit alloc prob}
\begin{split}
(R_1^*, \cdots, R_d^*) &= \argmin_{R_1, \cdots, R_d} \sum_{i=1}^{d} \sigma_i,	\\
&\text{subject to:}	\quad \sum_{i=1}^{d} R_i = R,
\end{split}
\end{equation}  
where $R_i$ is the number of bits allocated to $i$-th dimension. For solving the above minimization problem, $\{\sigma_i\}$ be calculated first.



To relate $\{\sigma_i\}$ and $\{R_i\}$, we use a simple scalar quantizer for each dimension. Assume we have a scalar random variable $u \sim \prob(u)$ which we want to quantize it by $R$ bits. We create $2^R$ equally probable bins over the real axis and use bins centroids as the reproduction points at the decoder (receiver). To compress $u$, we send index of the bin that $u$ belongs to it by $R$ bits.

Assume we denote the $i$-th bin interval by $(a_i, a_{i+1})$ and its centroid by $c_i$, then
\begin{align}
c_i &= 2^R \int_{a_i}^{a_{i+1}} u \prob(u) du, \qquad i = 1, \cdots, 2^R,
\end{align}
where $\{a_i\}$ are determined so that $\int_{a_i}^{a_{i+1}} P(u) du = 2^{-R}$. If $u \sim \norm(0, \sigma_u^2)$, then
\begin{equation}
c_i = \frac{2^R \sigma_u}{\sqrt{2 \pi}} \left(e^{-a_i^2/2\sigma_u^2} - e^{-a_{i+1}^2/2\sigma_u^2}\right), \quad i = 1, \cdots, 2^R.
\end{equation}
If $\{\alpha_i\}$ is the set of bins boundaries for standard normal distribution, then the boundaries for $u$ is obtained by $a_i = \sigma_u \alpha_i$. Note that the $\{\alpha_i\}$ is only function of $R$. Hence, we can calculate $c_i$ using
\begin{equation} \label{eq:c_i}
c_i = \frac{2^R \sigma_u}{\sqrt{2 \pi}} \left(e^{-\alpha_i^2/2} - e^{-\alpha_{i+1}^2/2}\right), \quad i = 1, \cdots, 2^R.
\end{equation}
The above equation shows that the centroids of any normal distribution are simply obtained by multiplying the centroids of standard normal distribution by standard deviation $\sigma_u$.
 
The expected reconstruction error for this scalar quantizer is calculated as follows:
\begin{align}
e(\sigma_u^2, R) &= \int_{-\infty}^{\infty} (u - \hat u)^2 P(u) du	\nonumber \\
&= \sum_{i=1}^{2^R} \int_{a_i}^{a_{i+1}} (u - c_i)^2 P(u) du \nonumber \\
&= \sum_{i=1}^{2^R} \int_{a_i}^{a_{i+1}} u^2 P(u) du + \sum_{i=1}^{2^R} c_i^2 \int_{a_i}^{a_{i+1}} P(u) du  \nonumber \\
&\qquad-2 \sum_{i=1}^{2^R} c_i \int_{a_i}^{a_{i+1}} u P(u) du	\nonumber \\
&= \sigma_u^2 + 2^{-R} \sum_{i=1}^{2^R} c_i^2 - 2 (2^{-R}) \sum_{i=1}^{2^R} c_i^2	\nonumber \\
&= \sigma_u^2 - 2^{-R} \sum_{i=1}^{2^R} c_i^2 \nonumber \\
&= \sigma_u^2 - \sigma_c^2,	 \label{eq:distortion per dim}	
\end{align}
where $\sigma_c^2$ is variance of the centroids.

In our problem setting, from \eqref{eq:dist of per symbol} and \eqref{eq:distortion per dim}, the distortion of $i$-th dimension is $\sigma_i(R_i) = e(\Lambda_{ii}, R_i)$, where $R_i$ is the number of bits allocated to the $i$-th dimension and $\Lambda_{ii}$ is the variance of $i$-th dimension of the transformed dataset $X'$. Now, we can solve the integer program in \eqref{eq: bit alloc prob}. Defining $\Delta \sigma_i = e(\Lambda_{ii}, R_i) - e(\Lambda_{ii}, R_i + 1)$ as the amount of decrease in distortion of $i$-th dimension by adding a new bit, then the greedy algorithm presented in Algorithm \ref{alg:bit alloc}, obtains the optimal bit allocation. In this algorithm, we allocate one bit in each iteration to the dimension with largest $\Delta \sigma$ and update the distortion of that dimension. We iterate this procedure until all available bits are allocated.
\begin{algorithm}[bt]
\begin{algorithmic}
	\caption{Greedy Algorithm for Bit Allocation}\label{alg:bit alloc}
	\State \textbf{Input}: Total bit rate R, Diagonal Matrix $\Lambda$
	\State \textbf{Output}: Dimensions bit rates: $R_1, \cdots, R_d$
	\For {$i = 1, \cdots, d$}
	\State $R_i = 0$
	\State $\Delta \sigma_i = e(\Lambda_{ii}, R_i) - e(\Lambda_{ii}, R_i + 1)$
	\EndFor
	\While {$R > 0$}
	\State $j = \argmax_i \Delta \sigma_i$
	\State $R_j \gets R_j + 1$
	\State  $\Delta \sigma_j = e(\Lambda_{jj}, R_j) - e(\Lambda_{jj}, R_j + 1)$
	\State $R \gets R - 1$
	\EndWhile\\
	\Return $R_1, \cdots, R_d$
\end{algorithmic}
\end{algorithm}

To show that the above greedy bit allocation algorithm is optimal, we rely on the fact that  $\Delta \sigma$ is a decreasing function of $R$. The plot of  $\Delta \sigma$ against the bit rate $R$ shows that it is indeed decreasing exponentially (\figurename{} \ref{fig:Delta-sigma plot}). Note that the figure is plotted for standard normal distribution and behaves similarly for any normal distribution. 

\begin{figure}[t]
	\centering
	\subfloat[]{\includegraphics[width=.49\linewidth]{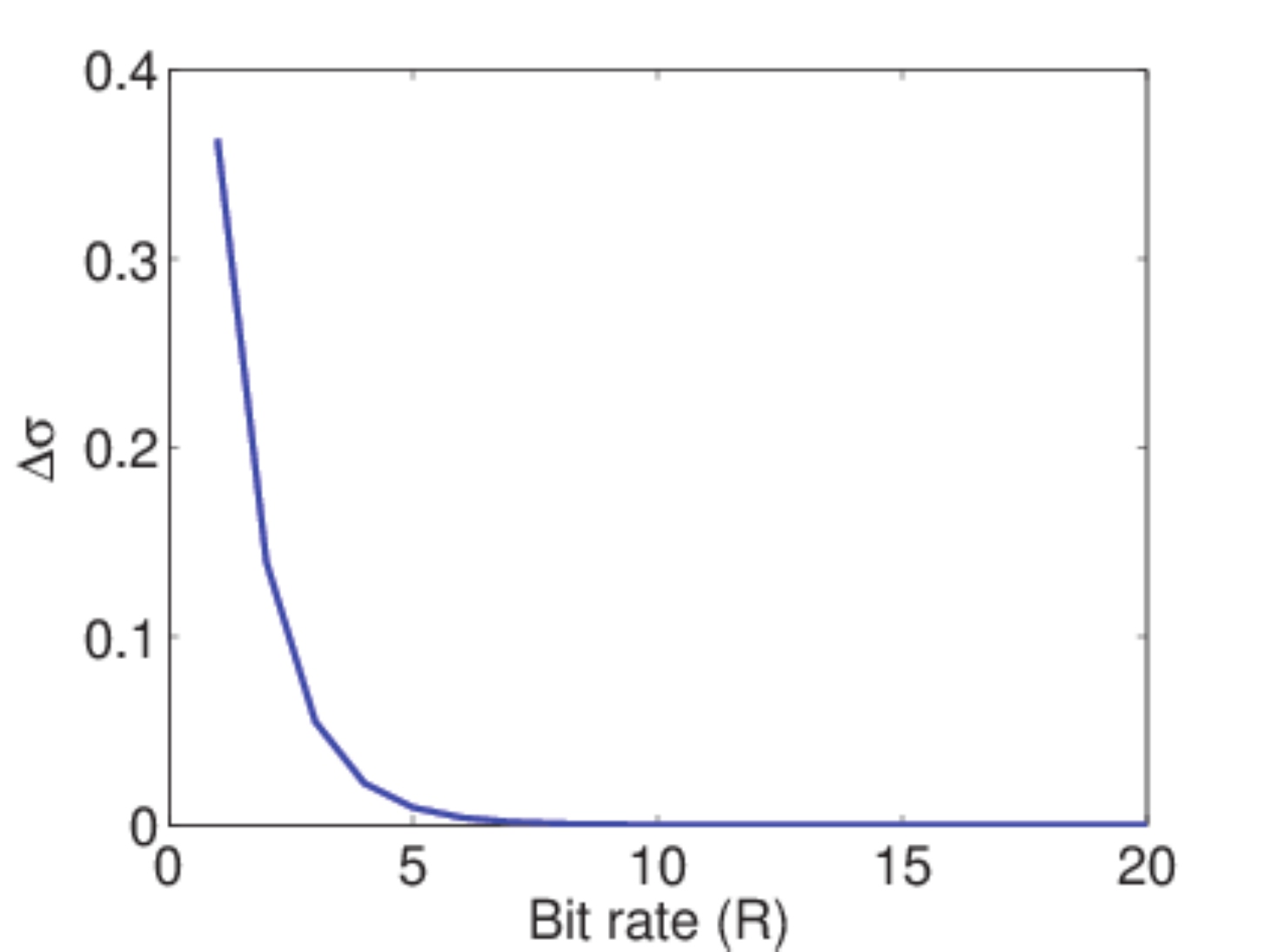}	\label{fig:delta sigma linear}}
	\subfloat[]{\includegraphics[width=.49\linewidth]{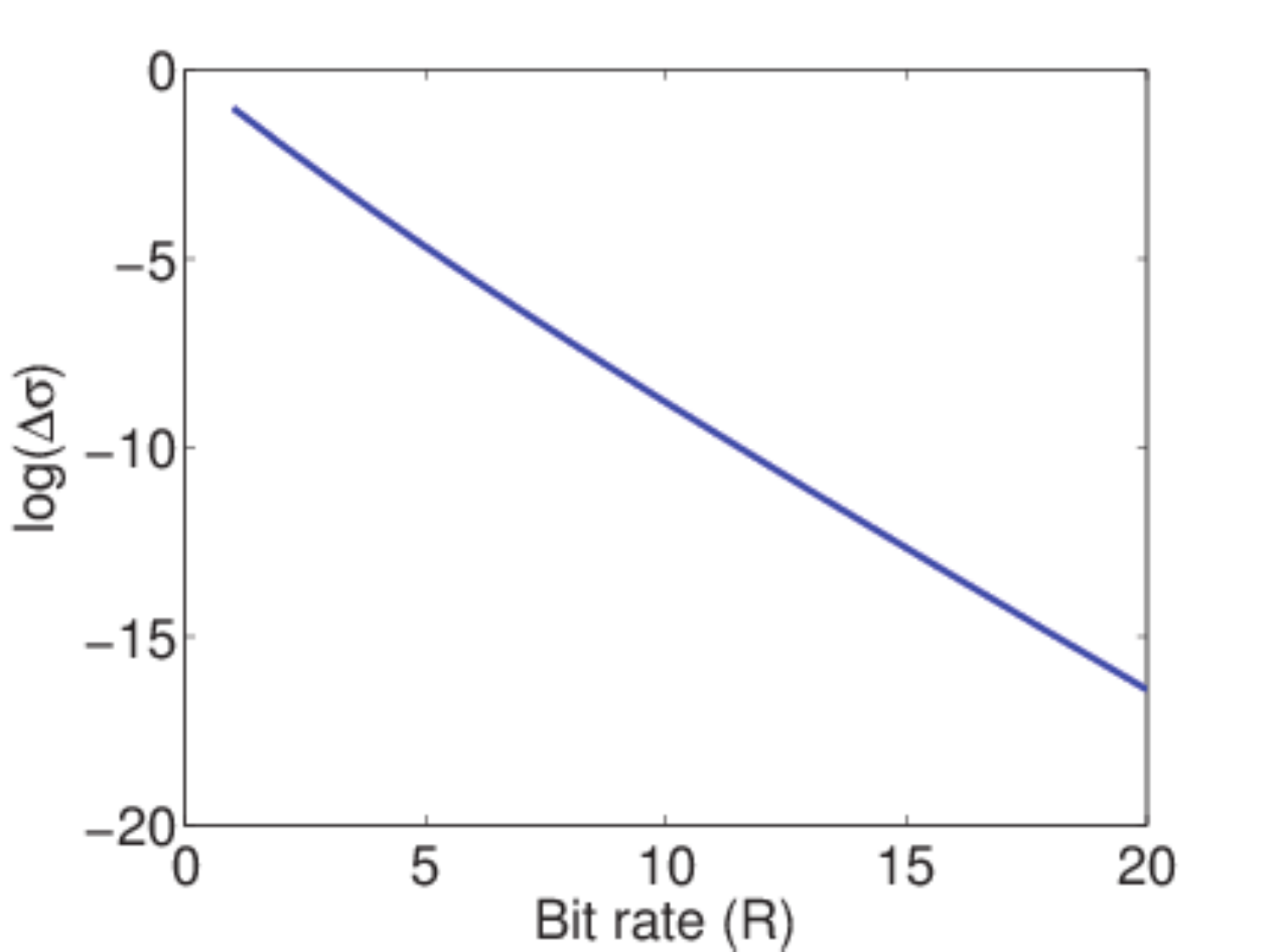}	\label{fig:delta sigma log}}
	\caption{The curve of $\Delta \sigma$ as a function of bit rate $R$ for standard normal distribution. (a) linear scale, (b) logarithmic scale.}
	\label{fig:Delta-sigma plot}
\end{figure}

We denote the optimal bit allocation by $S_{opt}=(R_1^*, \cdots, R_d^*)$ and the greedy solution by $(R_1^g, \cdots, R_d^g)$. We define the distance between the greedy and optimal solutions as follows:
\begin{equation}
r = \sum_{i=1}^{d} \vert R_i^* - R_i^g \vert.
\end{equation}
If the problem has multiple optimal solutions, we select a solution with the minimum distance value $r_{min}$. If $r_{min} = 0$, then the greedy solution is also optimal. Hence, we assume that $r_{min} > 0$ and show that it leads to a contradiction.

If $r_{min} > 0$, then there exist two bit rates $R_i^g$ and $R_j^g$ so that $R_i^g > R_i^*$ and $R_j^g < R_j^*$. 
Assume that in iteration $t$ of the greedy algorithm, the bit rate of dimension $i$ exceeds $R_i^*$, and the bit rate of dimension $j$ is $R_j^t$. Thus,
\begin{equation}\label{eq:delta i and j}
\Delta \sigma(R_i^*) \geq \Delta \sigma(R_j^t).
\end{equation} 
On the other hand, it is clear that $R_j^t \leq R_j^g < R_j^*$. Hence,
\begin{equation}\label{eq:delta j and j}
\Delta \sigma(R_j^t) \geq \Delta \sigma(R_j^g) \geq \Delta \sigma(R_j^*).
\end{equation}
From \eqref{eq:delta i and j} and \eqref{eq:delta j and j}, we have
\begin{equation}
\Delta \sigma(R_i^*) \geq \Delta \sigma(R_j^*).
\end{equation}
Thus, by removing a bit from $R_j^*$ and adding it to $R_i^*$ we obtain a new solution $S_{new}$ which its error value is less than or equal to error of $S_{opt}$. Since $S_{opt}$ is optimal, $S_{new}$ could not have smaller error value. But its distance with the greedy solution is $r_{min} - 2$ which is a contradiction.

The proposed per-symbol coding scheme has \order($R$) time complexity for encoding (using binary search algorithm over the codebook) and \order($1$) for decoding. However, its space complexity is \order($2^R$) that is required for storing the bins intervals and centroids (codebook). This method requires to transmit matrices $Q_x$ and $Q_y$ between machine \machine{x} and machine \machine{y}. Thus, the overall communication cost is \order($2 d^2 + R n)$ bits.

Although in the proposed per-symbol method the underlying distribution of samples of $X$ is considered to be normal, we can use it for other distributions. However, the normal distribution is used to make the dimensions independent and to prove the optimality of the greedy bit loading.  Experimental results in Section \ref{sec:experiments} reveal that the performance of this method for other distributions is desirable.

\subsection{Dimension Reduction} \label{sec:dimension reduction}
The third approach proposed in this paper for reducing communication cost for transmitting the dataset $X$ is to represent them in a lower dimensional space. In this section we propose a linear dimension reduction method that efficiently minimizes the inner product distortion. 
Interestingly, PCA, the conventional dimension reduction method, falls from optimality as it is shown in Section \ref{sec:experiments}.


Our goal is to construct a lower dimensional vector $\hat{x}_i$ for $x_i$ so that the error function \eqref{eq:err-func} be minimized.
In order to represent $x_i$ by $\hat{x}_i$ which lies in a lower dimensional space, $\hat{x}_i$ is defined as a weighted sum of some basis vectors as follows
\begin{equation}
	\label{eq:x_hat-non-matrix}
	\hat{x}_i = \sum_{k=1}^{m} z_{ik}u_k,
\end{equation}
where $\{u_k\}$ are orthogonal $d$-dimensional basis vectors so that $\Vert u_k \Vert ^ 2 = 1$, $z_{ik}$ is the weight of vector $u_k$ for $\hat{x}_i$, and $m$ is dimension of the new space. The equation \eqref{eq:x_hat-non-matrix} can be rewritten in vector-matrix form as 
\begin{equation} \label{eq:x_hat}
	\hat{x}_i = U z_i,
\end{equation}
where $z_i = [z_{i1}, \cdots, z_{i m}]^T$, and $U = [u_1, \cdots, u_m]$ is a $d \times m$ matrix that $u_k$ forms the $k$-th column, thus, $U^T U = I$. In fact, $z_i \in \real^m$ is the corresponding representation of $x_i$ in the $m$-dimensional space. We would like to minimize the error function \eqref{eq:err-func}, with respect to the matrix $U$ and vectors $z_i$, $i = 1, \cdots, n$.  More precisely, we want to find the solution of the following optimization problem:
\begin{equation} \label{eq:opt_dim_red}
	\begin{split}
		& \min_{z_1, \cdots, z_n, U} D	\\
		& \text{s.t.} 	\quad U^T U = I,
	\end{split}
\end{equation}
Following theorem gives the optimal values of this problem.
\begin{theorem} \label{thm:dim_red}
	The optimal basis vectors for minimizing $D$ in \eqref{eq:err-func} are the first $m$ right eigenvectors of matrix $S_x S_y$ corresponding to $m$ largest eigenvalues, where $S_x$ and $S_y$ are sample covariance matrices of $X$ and $Y$, respectively. The vector $z_i \in \mathbb{R}^m$ is obtained by
	\begin{equation}
		z_i = (U^T S_y U)^{-1} U^T S_y x_i, \qquad i = 1, \cdots, n,
	\end{equation}
	and the distortion value is the sum of $(d - m)$ smallest eigenvalues of $S_x S_y$.
\end{theorem}
\begin{proof}
	Substituting \eqref{eq:x_hat} in \eqref{eq:err-func} yields
	\begin{equation}\label{eq:D(U, z)}
		D = \tr(S_x S_y) + \frac{1}{n} \sum_{i=1}^{n} \left(z_i^T U^T S_y U z_i - 2 x_i^T S_y U z_i \right).
	\end{equation}
	We first optimize $D$ with respect to $z_i$ considering $U$ as a constant matrix. To this end, by taking the derivative with respect to $z_i$ and setting it to zero, we obtain
	\begin{equation} \label{eq:z_opt_value}
		z_i = A U^T S_y x_i, \qquad i \in \{1, \cdots, n\},
	\end{equation}
	where
	\begin{equation} \label{eq: A}
		A = (U^T S_y U)^{-1}.
	\end{equation}
	By substituting \eqref{eq:z_opt_value} in \eqref{eq:D(U, z)}, we obtain the error function $D$ as a function of $U$ as 
	\begin{equation} \label{eq:D(U)}
		D = \tr(S_x S_y) - \tr(A U^T S_y S_x S_y U).
	\end{equation}
	The first term in \eqref{eq:D(U)} is constant and can be ignored. Next, we obtain $U$ by introducing the Lagrange multiplier $B$ as follows:
	\begin{equation}\label{eq:lagrange_U}
		\mathcal{L} = \tr\left(A U^T S_y S_x S_y U\right) + \tr\left(B U^T U\right),
	\end{equation}
	where $B$ is an $m \times m$ symmetric matrix that comes from the constraint $U^T U = I$.
	Taking derivative from \eqref{eq:lagrange_U} with respect to $U$, yields
	\begin{equation}\label{eq:D_der_U}
		\frac{\partial D}{\partial U} = 2 S_y S_x S_y U A - 2 S_y U A U^T S_y S_x S_y U A + 2 U B.
	\end{equation}
	By setting it to zero and multiplying by $U^T$ from the left, yields $B = 0$. Hence, we have
	\begin{equation}
		2 S_y S_x S_y U A - 2 S_y U A U^T S_y S_x S_y U A = 0.
	\end{equation}
	Multiplying by $S_y^{-1}$ and $A^{-1}$ from left and right, respectively, yields
	\begin{equation} \label{eq:temp 3}
		S_x S_y U - U A U^T S_y S_x S_y U = 0.
	\end{equation}
	It is easy to show that replacing $S_x S_y U$ with $U \Lambda$ where $\Lambda$ is a diagonal matrix which includes $m$ eigenvalues of $S_x S_y$, satisfies \eqref{eq:temp 3}. In other words, if columns of $U$ are the right eigenvectors of $S_x S_y$, then it satisfies \eqref{eq:temp 3}. Thus, we would like to solve the following eigenvector problem:
	\begin{equation} \label{eq:U_solution}
		S_x S_y U = U \Lambda.
	\end{equation}
	By substituting \eqref{eq:U_solution} in \eqref{eq:D(U)}, we obtain
	\begin{equation}
		D = \tr(S_x S_y) - \tr(\Lambda).
	\end{equation}
	Clearly for minimizing $D$, the trace of $\Lambda$ should be maximized, thus, the $m$ right eigenvectors of $S_x S_y$ that correspond to the $m$ largest eigenvalues must be selected as the columns of $U$.
\end{proof}
The above theorem shows that the optimal dimension reduction can be obtained from factorization of $S_x S_y$. Therefore, $S_y$ is first transmitted to machine \machine{x}.
Transmitting $S_y$ from machine \machine{y} requires communicating \order($d^2$) bits. After dimension reduction, transmitting $\{z_i\}$ and $U$ to the machine \machine{y} requires \order($m n + m d$) bits, thus, the overall communication cost is \order($d^2 + m n + m d$) bits.
Theorem \ref{thm:dim_red} shows that if the dimensions in $Y$ dataset are highly correlated, then the error value is significantly lower than PCA and conversely for loosely correlated dimensions the method is close to the PCA so that if $S_y = I$ then it will be completely equivalent to PCA.

One more interesting observation is that the solution to dimension reduction in Theorem \ref{thm:dim_red} is very similar to that of Theorem \ref{thm:compression} in the sense that both rely on $S_x S_y$ (or $Q_x Q_y$) for choosing their best strategy.

It is worth mentioning that if $S_x \simeq S_y$  (e.g. $X$ and $Y$ have identical distributions) there is no need to transmitting $S_y$, thus the communication cost reduces by \order($d^2$) bits. In this case, the machine \machine{x} can perform the proposed dimension reduction by estimating $\hat{S}_y = S_x$. But in this case, our proposed method is equivalent to the PCA. Because, eigenvectors of $S_x^2$ are identical to that of $S_x$ and the eigenvalues are squared. Since $S_x$ is positive definite, order of eigenvalues for $S_x$ and $S_x^2$ are identical. This shows that our proposed method for dimension reduction outperforms the PCA when the local datasets on the machines differ from each other. This will be shown experimentally in Section \ref{sec:experiments}.



\section{Distributed Learning of GPs} \label{sec:distr gp}
In this section, we make use of  the three methods proposed in Section \ref{sec:trans schemes} to learn  GP models over distributed datasets. Two communication models are considered: \emph{single-center} and \emph{broadcast} models.  In the first model, we assume that there is a centeral machine in the network which is responsible for learning the final GP model. In the second model, we assume that each machine in the network has a local GP model which is obtained from local dataset and the received data from all other machines. In this model, we also assume that each machine broadcasts its encoded data and all the machines receive the transmitted data.  

Suppose the dataset $\set{D} = \cup_{i=1}^m X_i$ is stored in $m$ machines \machine{1}, $\cdots$, \machine{m} such that $X_i$ is stored in machine \machine{i}. For convenience and without loss of generality, we have assumed that each local dataset $X_i$ is of size $n$. Hence, $\set{D}$ includes $N = m n$ observations in total. 
%
%
To learn a GP model in distributed systems, we first estimate the gram matrix which can be partitioned as follows:
\begin{equation} \label{eq:gram-mat-structure}
	G = 
	\begin{bmatrix}
		G_{11} & G_{1 2} & \cdots & G_{1 m}	\\
		G_{21} & G_{2 2} &\cdots & G_{2 m} \\
		& & \ddots & \\
		G_{m1} & G_{m 2 }& \cdots & G_{m m}
	\end{bmatrix}.
\end{equation}
Clearly, \machine{i} can calculate the gram matrix of its local dataset which is the sub-matrix $G_{ii}$ in \eqref{eq:gram-mat-structure}. However, the other parts of $G$ should be calculated cooperatively among the machines. In the PoE-based models (BCM and rBCM) the off-diagonal sub-matrices are considered to be zero. In other words, they assume that the marginal likelihood $P(\vect{y} | X, \theta)$ can be factorized as a product of local marginal likelihoods. The assumption of independence is not valid in many practical situations. Therefore, we aim at learning the full $G$ with the help of data communication between machines. 

We propose an algorithm based on Nystr\"{o}m method for gram matrix approximation. 
%
%
In this approximation, it is assumed that only $K < N$ rows (columns) of the gram matrix is given and the other rows (columns) must be approximated. Without loss of generality, we assume that the first $K$ rows of the gram matrix is available and the rest $N-K$ rows is approximated based on the first $K$ rows. Partitioning $G$ as 
\begin{equation}
	\begin{bmatrix}
		G_{K K} & G_{K (N-K)}	\\
		G_{(N-K) K} & G_{(N-K)(N-K)}
	\end{bmatrix}
\end{equation}
and denoting the first $K$ rows of $G$ by $G_{K N}$, the Nystr\"{o}m approximation of $G$ is defined as
\begin{equation}\label{eq:nystrom-approx}
	\hat G = G_{N K} G_{K K}^{-1} G_{K N},
\end{equation}
where $G_{N K}$ is the transpose of $G_{K N}$ \cite{rasmussen2006gaussian}.
 
It can be shown that the two matrices $G$ and $\hat G$ agree on the first $K$ rows. In fact, the error in estimating $G$ comes from  $\hat{G}_{(N-K) (N-K)}$. It can be shown that the error in $\hat{G}_{(N-K) (N-K)}$  is  the Schur complement of $G_{KK}$ \cite{rasmussen2006gaussian}. It worth mentioning that, Snelson and Ghahramani in \cite{snelson2005sparse} showed that the Nystr\"{o}m approximation can be improved by making it exact in the diagonal.

\subsection{Single-Center Model}
In the case of single-center model, we assume that one of the machines in the network, say \machine{1}, is the center node responsible for learning of the GP model. First,  the center machine sends its local covariance matrix to other machines. 
Having known the covariance matrix of the center machine, each machine uses one of the coding schemes proposed in Section \ref{sec:trans schemes} to transmit its distorted data to the center. Finally, the center machine calculates the total gram matrix and learns hyper-parameters of the model. 

For example, if \machine{1} is the center node, it calculates its local gram matrix $G_{11}$ and local sample covariance matrix $S_1$. Then, it sends $S_1$ to other machines and each machine encodes its local dataset and transmits it to \machine{1}. Machine \machine{1} uses the received distorted data to calculate the gram matrices $G_{12}, G_{13}, \cdots, G_{1m}$. Then, the full gram matrix $G$ is approximated by \eqref{eq:nystrom-approx}. Finally, using one of the gradient based optimization methods, the hyper-parameters of the model are learned, c.f. \cite{rasmussen2006gaussian}. The prediction procedure is also performed entirely at the center machine.

\subsection{Broadcast Model}
In broadcast networks, each machine sends its local covariance matrix to other machines. Then, each machine broadcasts its local dataset to all other machines incorporating the methods proposed in  Section \ref{sec:trans schemes}. In particular, at machine $i$, sum of the local covariance matrices of other machines (i.e. $\sum_{j = 1, j \neq i}^m S_j$) is used to obtain the best coding scheme for the proposed methods in Section \ref{sec:trans schemes}.

Despite the single-center model, the gram matrix is approximated locally  at each machine in the broadcasting model. Then, a predictive distribution can be obtained by combining these local GPs. In the broadcast model, the procedure that each machine performs is similar to the center machine in the single-center model.
For example, consider the machine \machine{1}; this machine can calculate its local gram matrix $G_{1 1}$. The sub-matrices $G_{1 2}, G_{1 3}, \cdots, G_{1 m}$ can be estimated using one of the methods presented in section \ref{sec:trans schemes}. Thus,  \machine{1} can construct the first $n$ rows of the gram matrix $G$ and obtain the whole matrix  by the Nystr\"{o}m approximation. The same procedure is done in other machines. After all, each machine has an approximate version of the full GP model. 

In the prediction phase, a new test observation is sent to all machines. Each machine obtains a normal predictive distribution parameterized as  $\mathcal{N}_i(\mu_i, \Sigma_i)$ for the target value of the input observation. Next, the local predictive distributions are aggregated to obtain the final predictive distribution at a fusion machine. The fusion machine can be one of the machines or a separate machine. It is also possible that all machines obtain the final predictive distribution if they access to all $\mu_i$'s and $\Sigma_i$'s. For merging local predictive distributions $\mathcal{N}_i(\mu_i, \Sigma_i)$ in the fusion machine we use the following objective function to obtain the final mean and covariance matrix:
\begin{equation} \label{eq:fusion obj func}
	(\mu^*, \Sigma^*) = \arg\min_{\mu, \Sigma} \sum_{i = 1}^{m}	 KL\bigg(\mathcal{N}(\mu_i, \Sigma_i) \parallel \mathcal{N}(\mu, \Sigma)\bigg).
\end{equation}
In \cite{liu2014distributed}, it is shown that the above objective function obtains the global maximum likelihood estimator for the situation in which $\mu_i$ and $\Sigma_i$ are maximum likelihood estimators of local dataset $X_i$. Although in our problem setting the $\mu_i$'s and $\Sigma_i$'s are not local maximum likelihood estimators, the objective function \eqref{eq:fusion obj func} intuitively is reasonable as a fusion criterion. However, by solving \eqref{eq:fusion obj func} we obtain the optimal values for $\mu$ and $\Sigma$ as follow
\begin{align}
	\mu^* &= \frac{1}{m} \sum_{i=1}^{m} \mu_i	,	\label{eq:mu star}	\\
	\Sigma^* &= \frac{1}{m} \sum_{i=1}^{m} \Sigma_i + (\mu^* - \mu_i) (\mu^* - \mu_i)^T \label{eq:cov star}.
\end{align}
Thus, each machine sends its calculated mean $\mu_i$ and covariance matrix $\Sigma_i$ to the fusion center, say machine \machine{1}, and it calculates the final predictive distribution using \eqref{eq:mu star} and \eqref{eq:cov star}.


\section{Experiments} \label{sec:experiments}
In this section, we evaluate our proposed coding schemes presented in Section \ref{sec:trans schemes} over synthetic and real datasets. After comparing the distortion of the coding schemes in recovering the inner products, we evaluate their performance in training GP regression models over distributed datasets.

\figurename{} \ref{fig:distortion-cc} shows the distortion value in \eqref{eq:err-func} against the communication rate for the  proposed methods in Section \ref{sec:trans schemes}. The dataset is generated synthetically by sampling from  a $20$-dimensional Gaussian distribution by a random covariance matrix. The dataset is distributed between two machines.
As can be seen, the distortion of the per-symbol compression scheme is close to the lower bound on the distortion which is  obtained by Theorem \ref{thm:compression}; and is significantly lower than the proposed dimension reduction method. For the distortion curve of the dimension reduction method, we have assumed that each dimension in lower space is represented by $16$ bits without any quantization distortion. Thus, the distortion for dimension reduction is only due to loss of information by the lost dimensions.  

From the experiment, it is evident that after expending $70$ bits per symbol or $3.5$ bits per dimension, the distortion is almost zero for the optimal distortion curve. This value becomes $5$ bits per dimension for the per-symbol compression scheme and $12.5$  bits for the dimension reduction scheme. 

\begin{figure}[tb]
	\centering
	\includegraphics[width=.8 \linewidth]{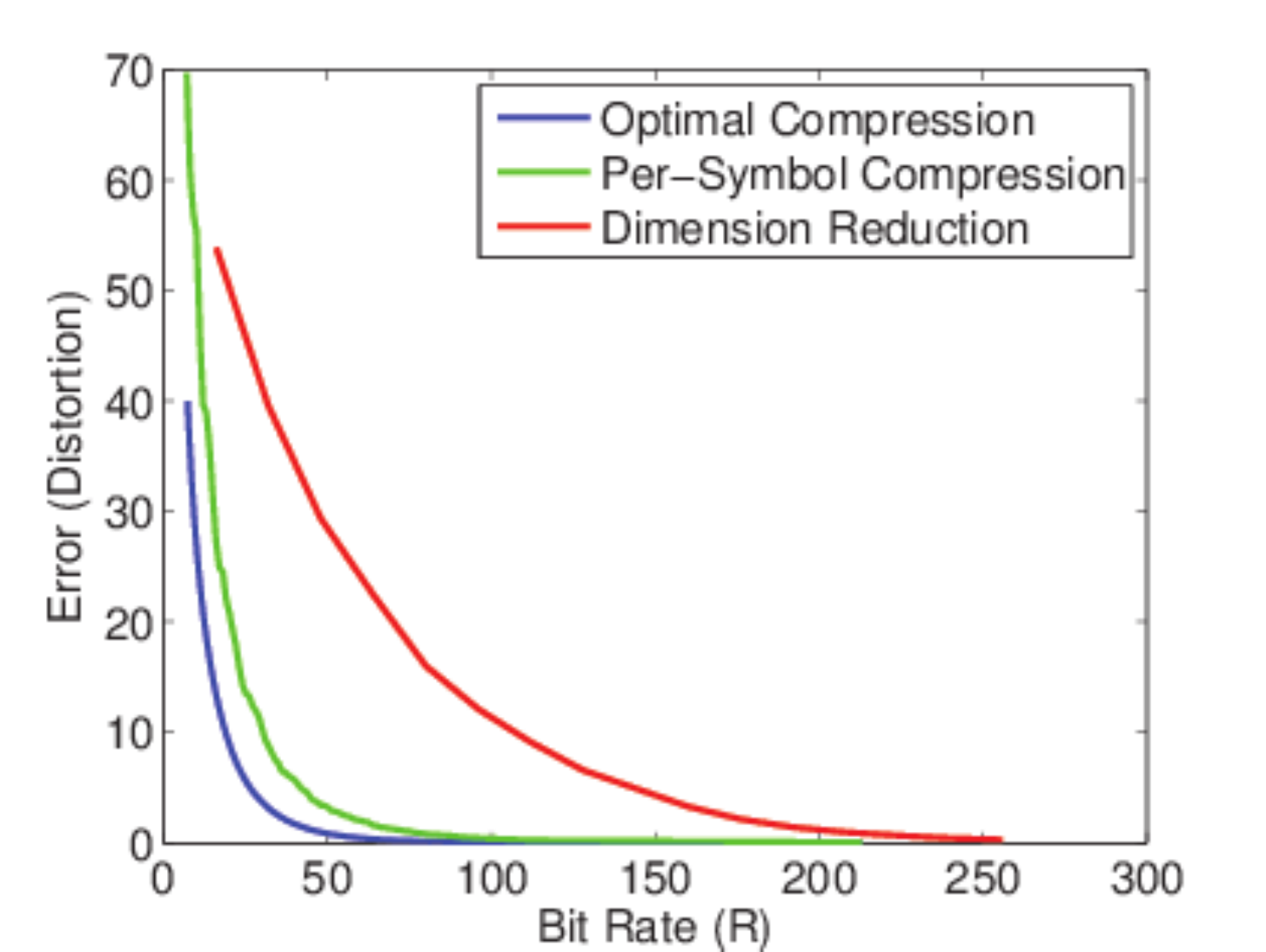}
	\caption{The distortion curve of the optimal (lower bound), per-symbol compression and the dimension reduction methods for $20$-dimensional Gaussian distribution.}
	\label{fig:distortion-cc}
\end{figure}

In \figurename{} \ref{fig:mnist-dim-red}, we have compared two different dimension reduction methods: our proposed scheme and the one based on PCA. In the first experiment, \figurename{} \ref{fig:gauss non-uniform dim red} we have generated two datasets $X$ and $Y$ from two multivariate Gaussian distributions with different covariance matrices. Whereas, in the second experiment, \figurename{} \ref{fig:gauss uniform dim red} we have drawn both datasets from single Gaussian distribution. In the first case, since the covariance matrices are different, our method has lower distortion than PCA. But in the second case, the two methods have approximately the same distortion values. The reason is that in this situation both local covariance matrices $S_x$ and $S_y$ are approximately equal and close to the true covariance matrix (specially for large datasets). Thus, the basis vectors obtained by PCA and our method are identical. However, when the sample covariance matrices $S_x$ and $S_y$ are far from each other, our method outperforms PCA. 

The comparison of the two schemes on real datasets is presented in \figurename{} \ref{fig:mnist non-uniform dim red} and \figurename{} \ref{fig:mnist uniform dim red} where the popular hand-written digits dataset known as MNIST is used. In the first experiment (\figurename{} \ref{fig:mnist non-uniform dim red}), the images of digits $6$ and $7$ are stored in the first and second machines, respectively. In the second experiment (\figurename{} \ref{fig:mnist uniform dim red}),  the images are split uniformly between two machines. \figurename{} \ref{fig:mnist non-uniform dim red} shows that under  the first condition our method outperforms PCA. On the other hand, \figurename{} \ref{fig:mnist uniform dim red} shows that under the second condition, the two methods have equal distortion.


\begin{figure}[t]
	\centering
	\subfloat[]{\includegraphics[width=.5\linewidth]{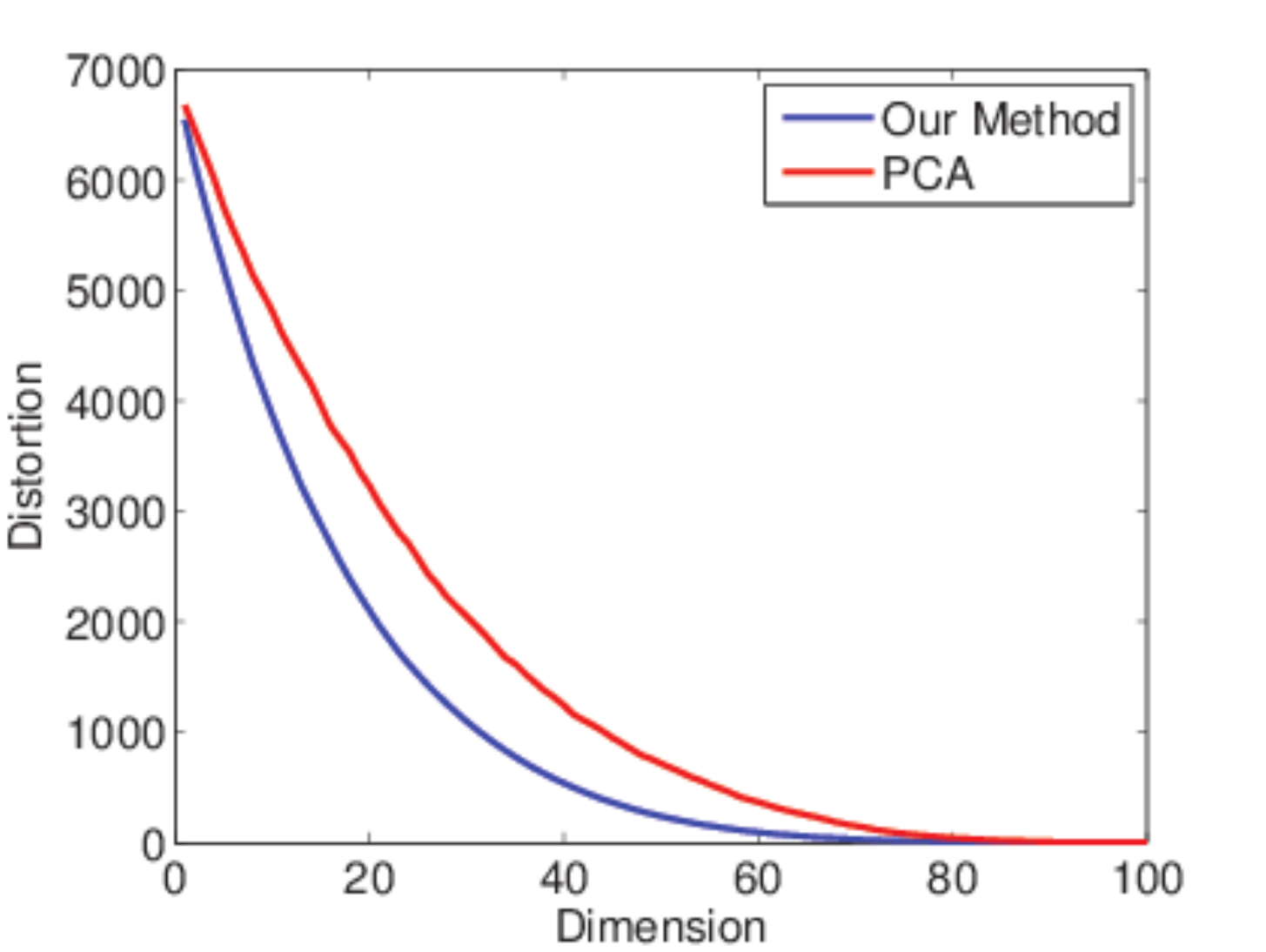} %
	\label{fig:gauss non-uniform dim red}}
	\subfloat[]{\includegraphics[width=.5\linewidth]{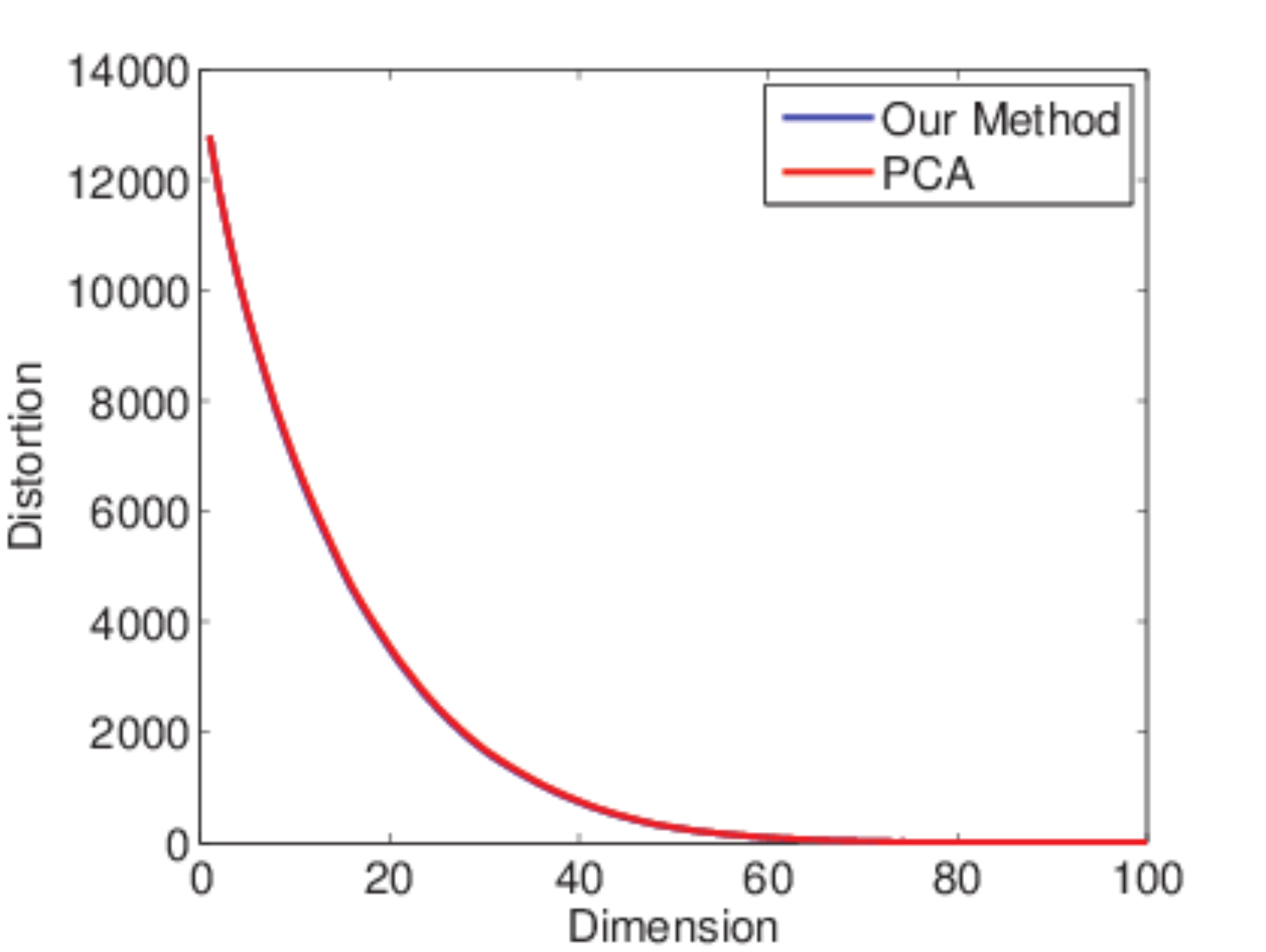} %
	\label{fig:gauss uniform dim red}}\\
	\subfloat[]{\includegraphics[width=.5\linewidth]{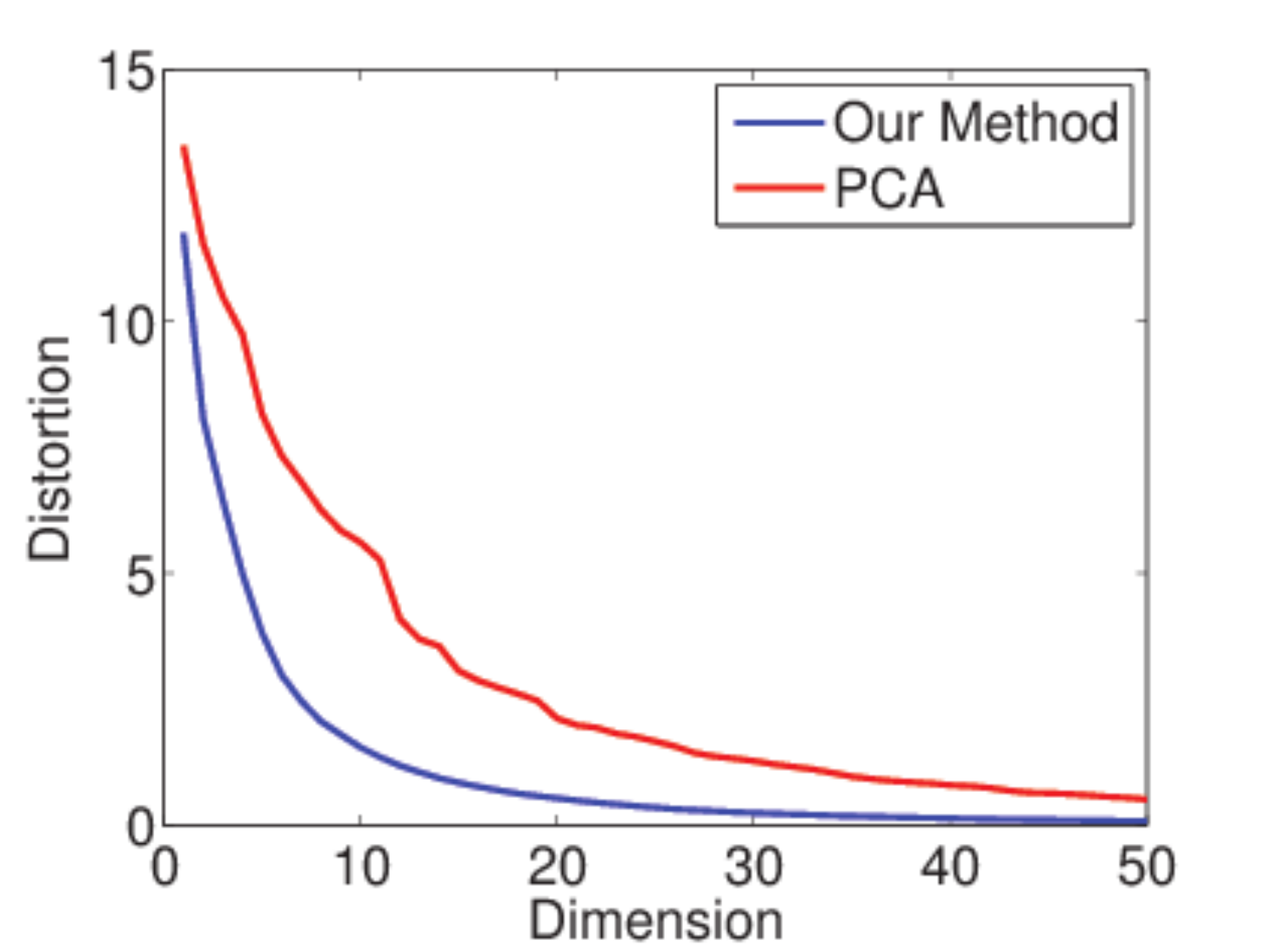} %
	\label{fig:mnist non-uniform dim red}}
	\subfloat[]{\includegraphics[width=.5\linewidth]{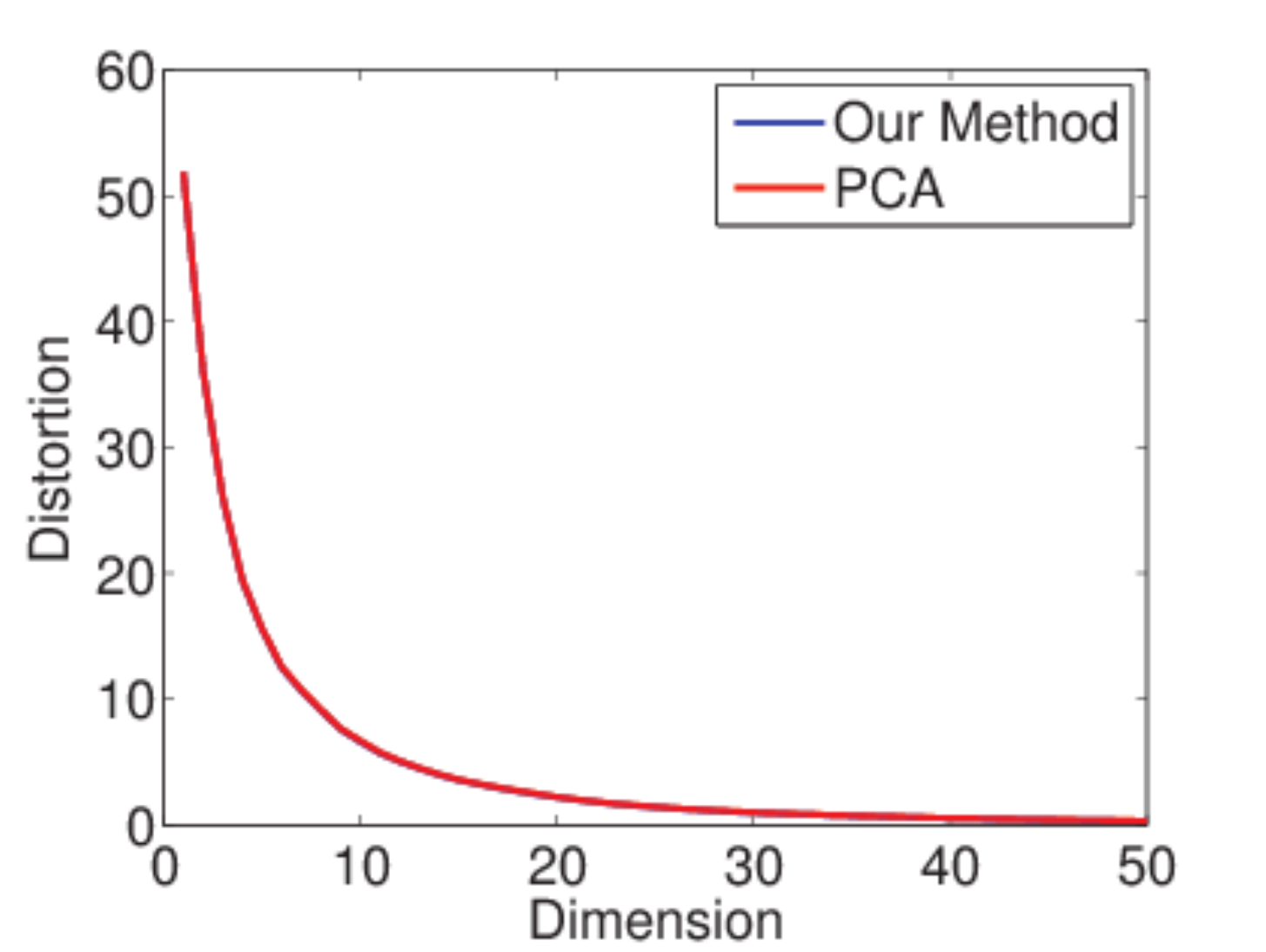} %
	\label{fig:mnist uniform dim red}}
	\caption{Comparison of distortions for PCA and the proposed dimension reduction. (a) Gaussian data with different covariance matrices in each machine. (b) Gaussian data with identical covariance matrices in the machines. (c) Images of MNIST for digit $6$ in the first machine and digit $7$ in the other. (d) Distributing images of MNIST for digits $6$ and $7$ uniformly between two machines.}
	\label{fig:mnist-dim-red}
\end{figure}

A simple $1$-dimensional database is used for evaluation of GP training on quantized datasets. \figurename{} \ref{fig:1d-gpr} shows the GP trained on a synthetic $1$-dimensional dataset of size $N = 200$. In this experiment, the whole dataset is quantized by rates $R=1$ to $R=8$ bits. The quantization method explained in Section \ref{sec:approx compression} is used for this experiment. In the case of $R = 1$, the true GP has a local maximum around $-5$, while the quantized GP has a local minimum at that region. This phenomenon is observed again around $5$ where true GP has local minimum while the quantized GP has local maximum. Thus, the divergence of posterior distribution of true and quantized GPs for $R=1$ is very high. Also as can be seen from \figurename{} \ref{subfig:1d-1bit} the standard deviation of quantized GP is greater than true GP in the region of data. This experiment shows that GP training using very low quality train data may not be favorable. 

In despite of the case $R = 1$, the quantized GP for $R=2$ is more close to the true GP. Here we have no reverse peaks between the true and quantized GPs. However, it seems that the divergence between the two GPs is relatively high. For rate $R = 3$, the mean function is fairly close to the true version, but the standard deviation in the region of training inputs is slightly larger than the true one. For rates $R \geq 6$, both mean and standard deviation of quantized GP is very close to the true GP. This experiment shows that by consuming a few bits for each data sample, we can reach to the result of full GP model, but if the bit rate is extremely low, say $1$ bit/sample in this experiment, the result may be unappealing.

\begin{figure*}[tb]
	\subfloat[$R=1$]{\includegraphics[width=.24\linewidth]{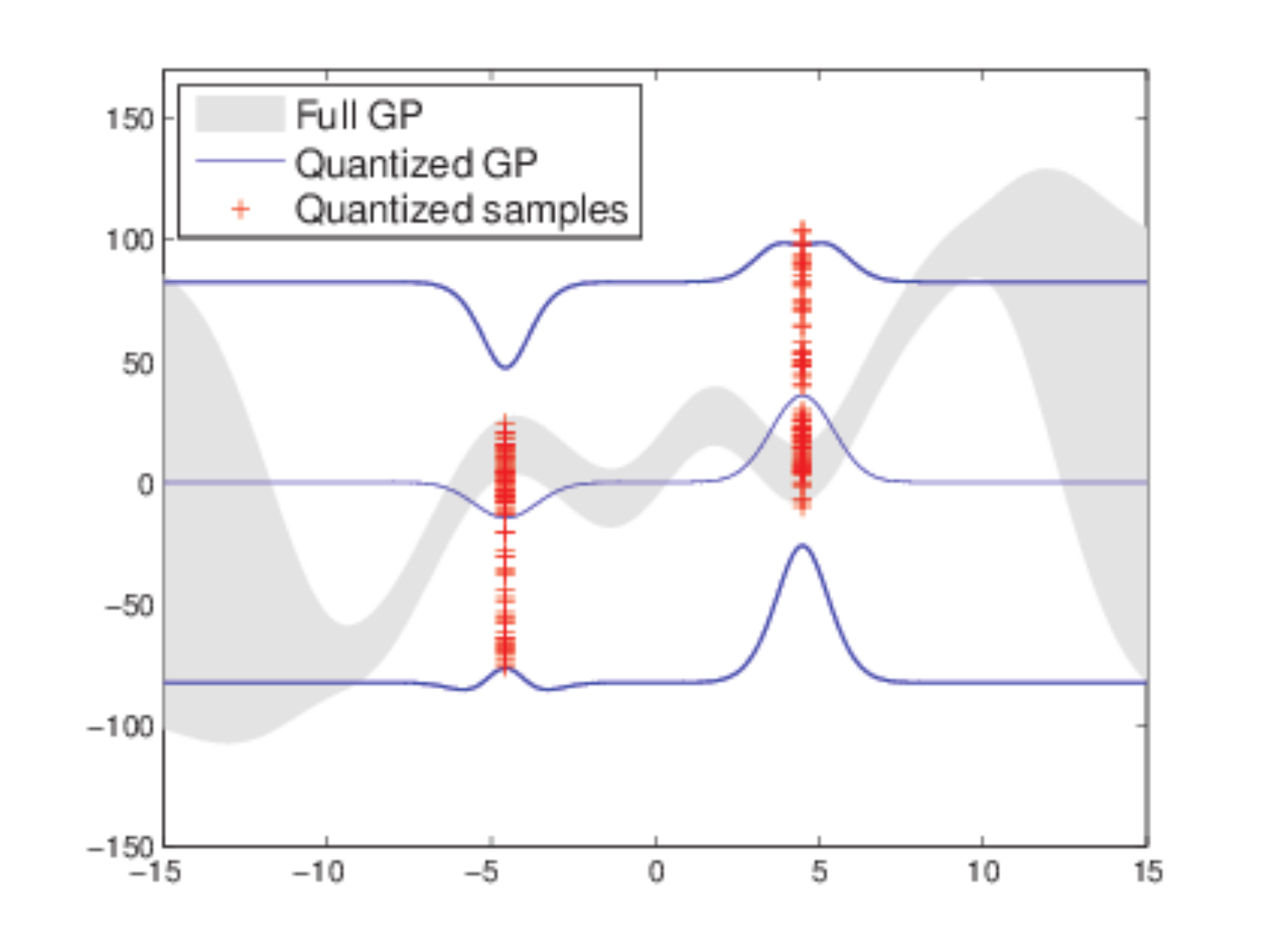}%
	\label{subfig:1d-1bit}}
	\hfil
	\subfloat[$R=2$]{\includegraphics[width=.24\linewidth]{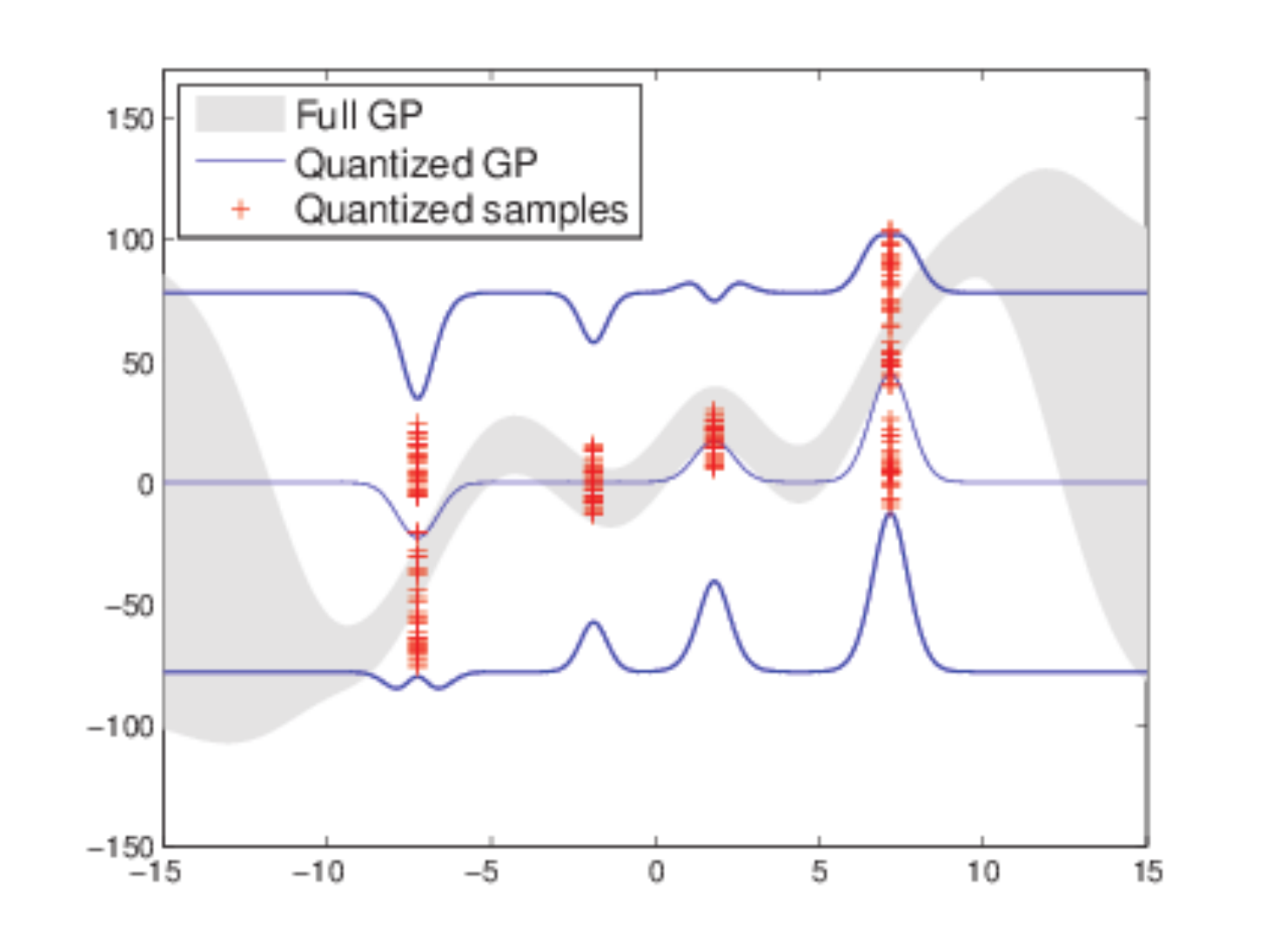}%
		\label{subfig:1d-2bit}}
	\hfil
	\subfloat[$R=3$]{\includegraphics[width=.24\linewidth]{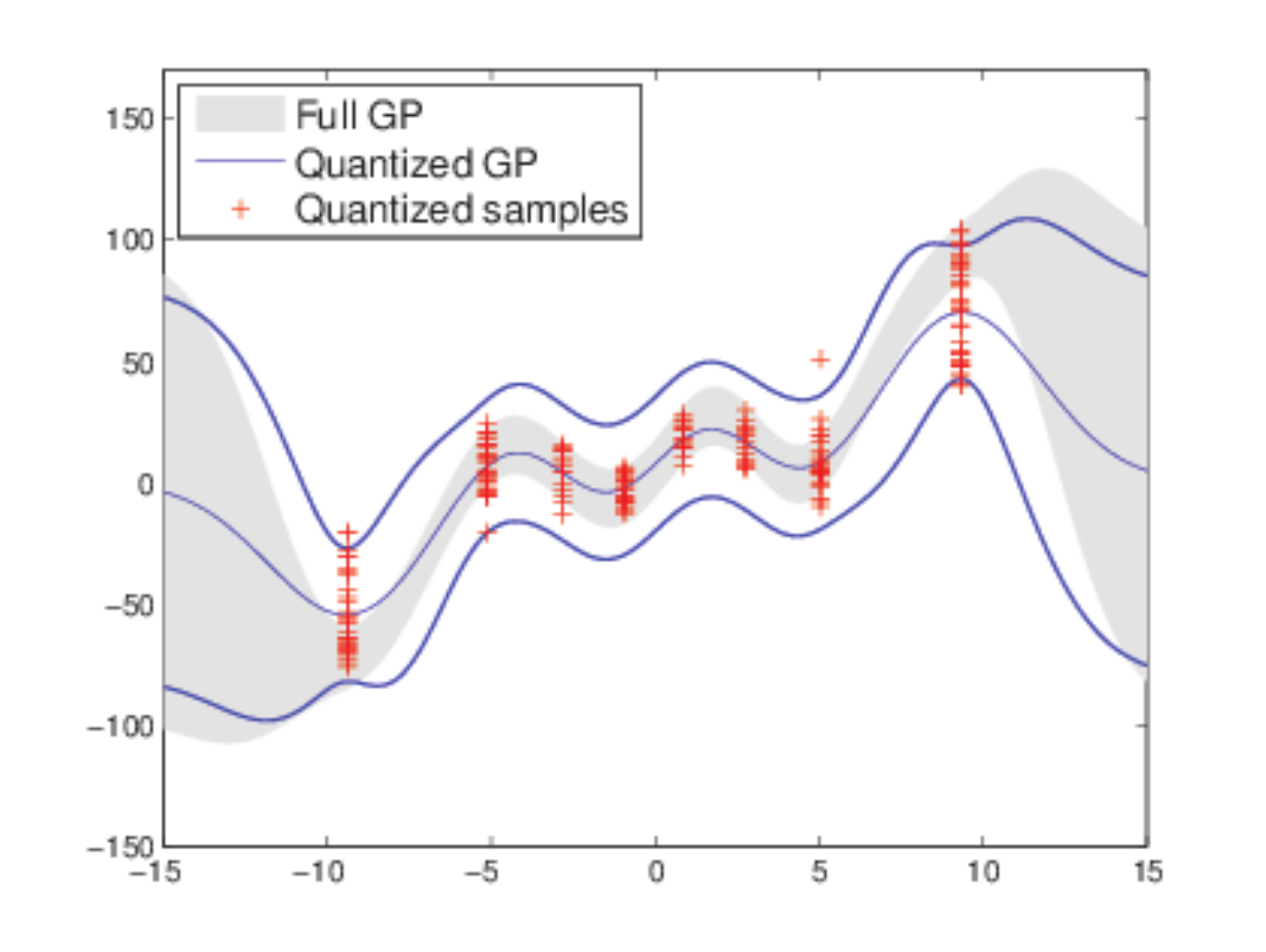}%
		\label{subfig:1d-3bit}}
	\hfil
	\subfloat[$R=4$]{\includegraphics[width=.24\linewidth]{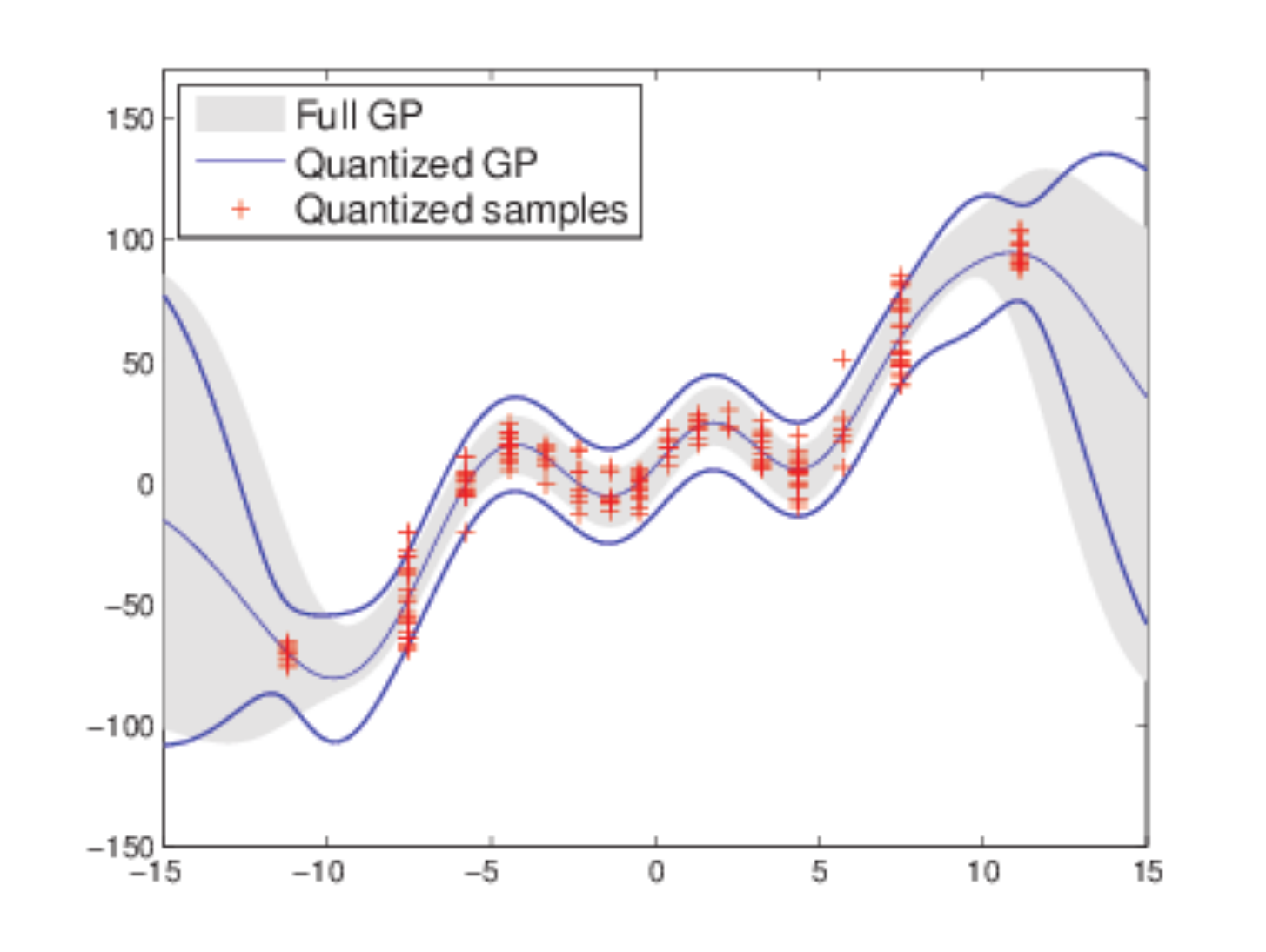}%
		\label{subfig:1d-4bit}}
	\hfil
	\subfloat[$R=5$]{\includegraphics[width=.24\linewidth]{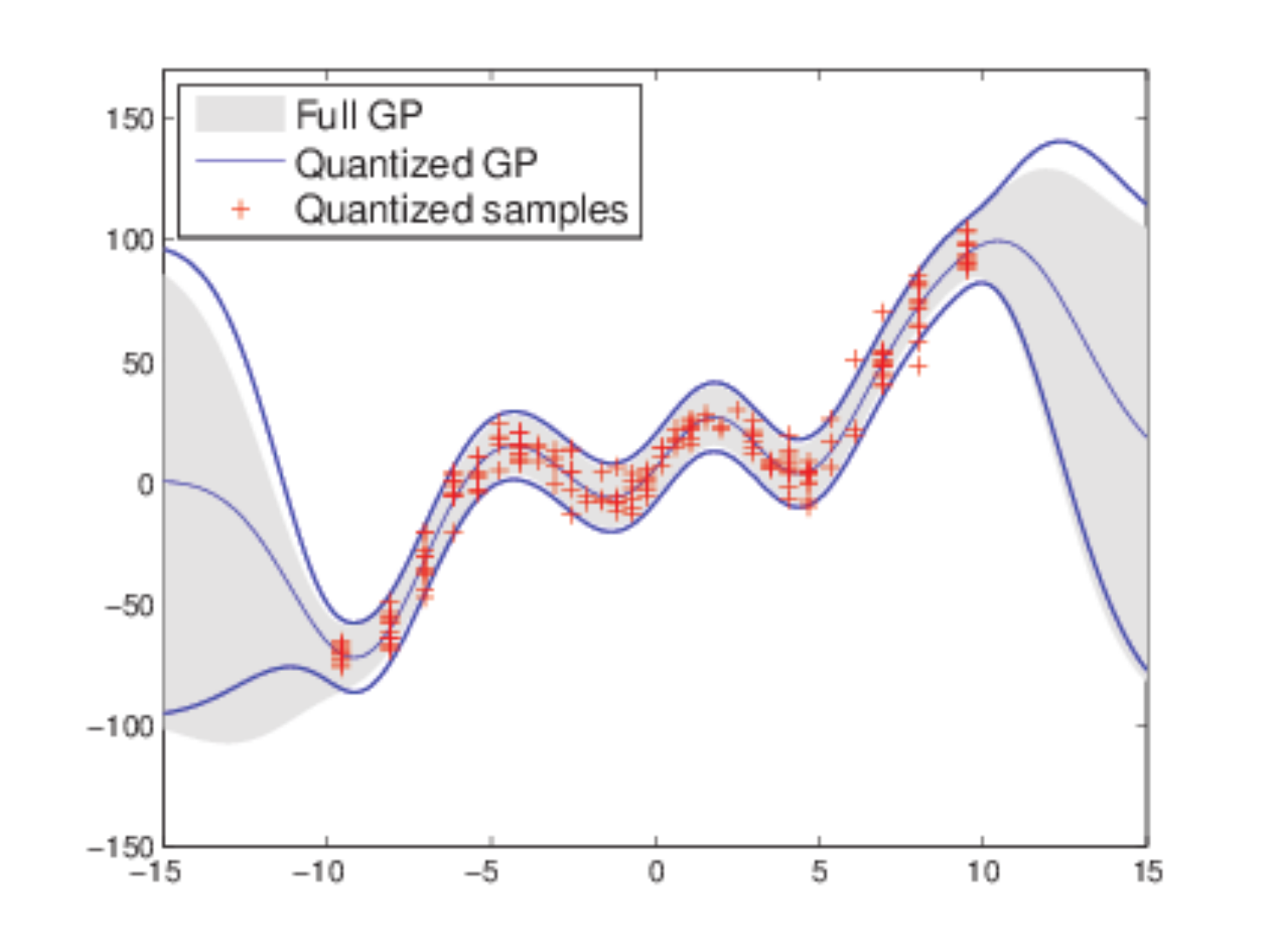} %
	\label{subfig:1d-5bit}}
	\hfil
	\subfloat[$R=6$]{\includegraphics[width=.24\linewidth]{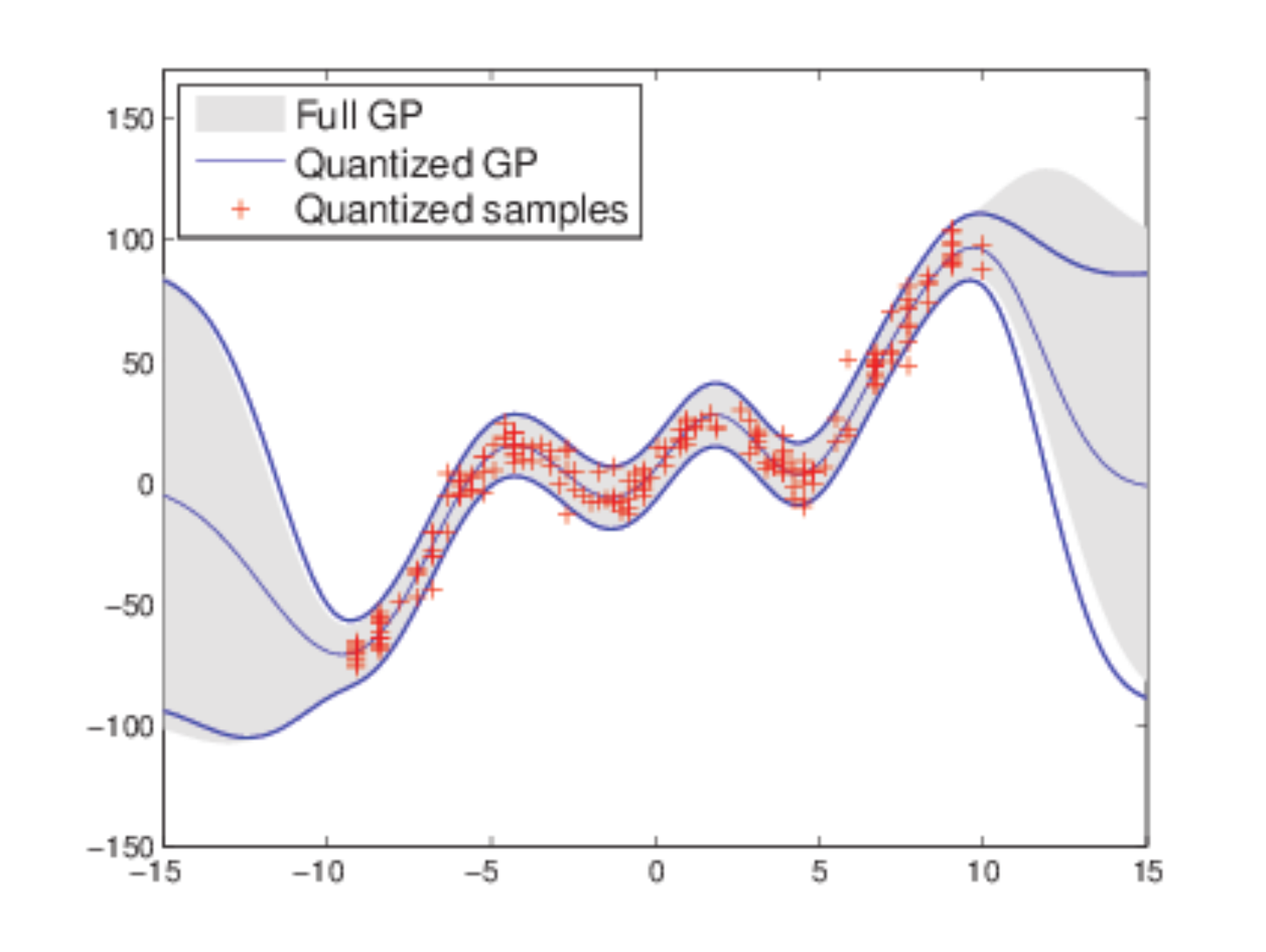}%
		\label{subfig:1d-6bit}}
	\hfil
	\subfloat[$R=7$]{\includegraphics[width=.24\linewidth]{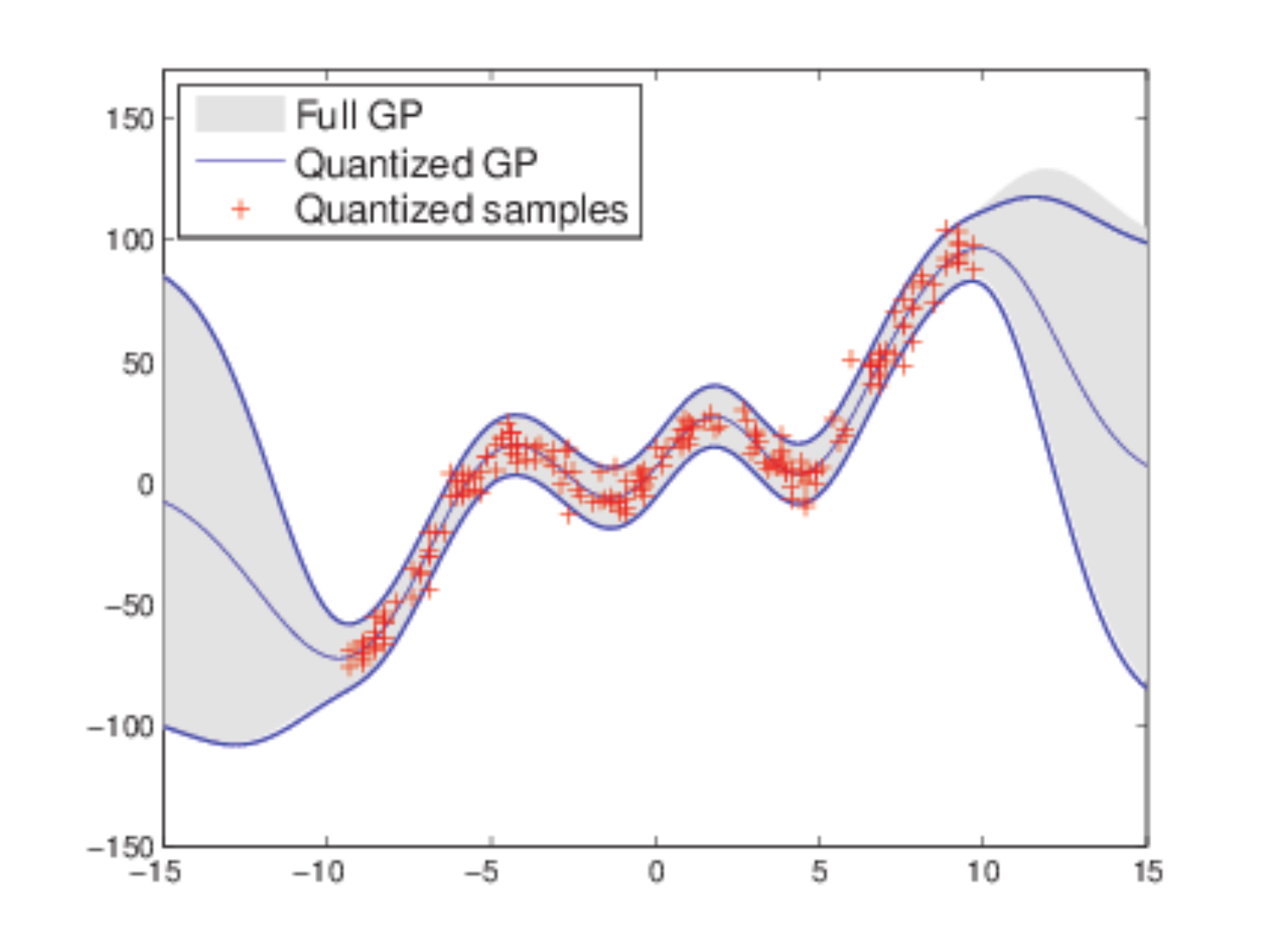}%
		\label{subfig:1d-7bit}}
	\hfil
	\subfloat[$R=8$]{\includegraphics[width=.24\linewidth]{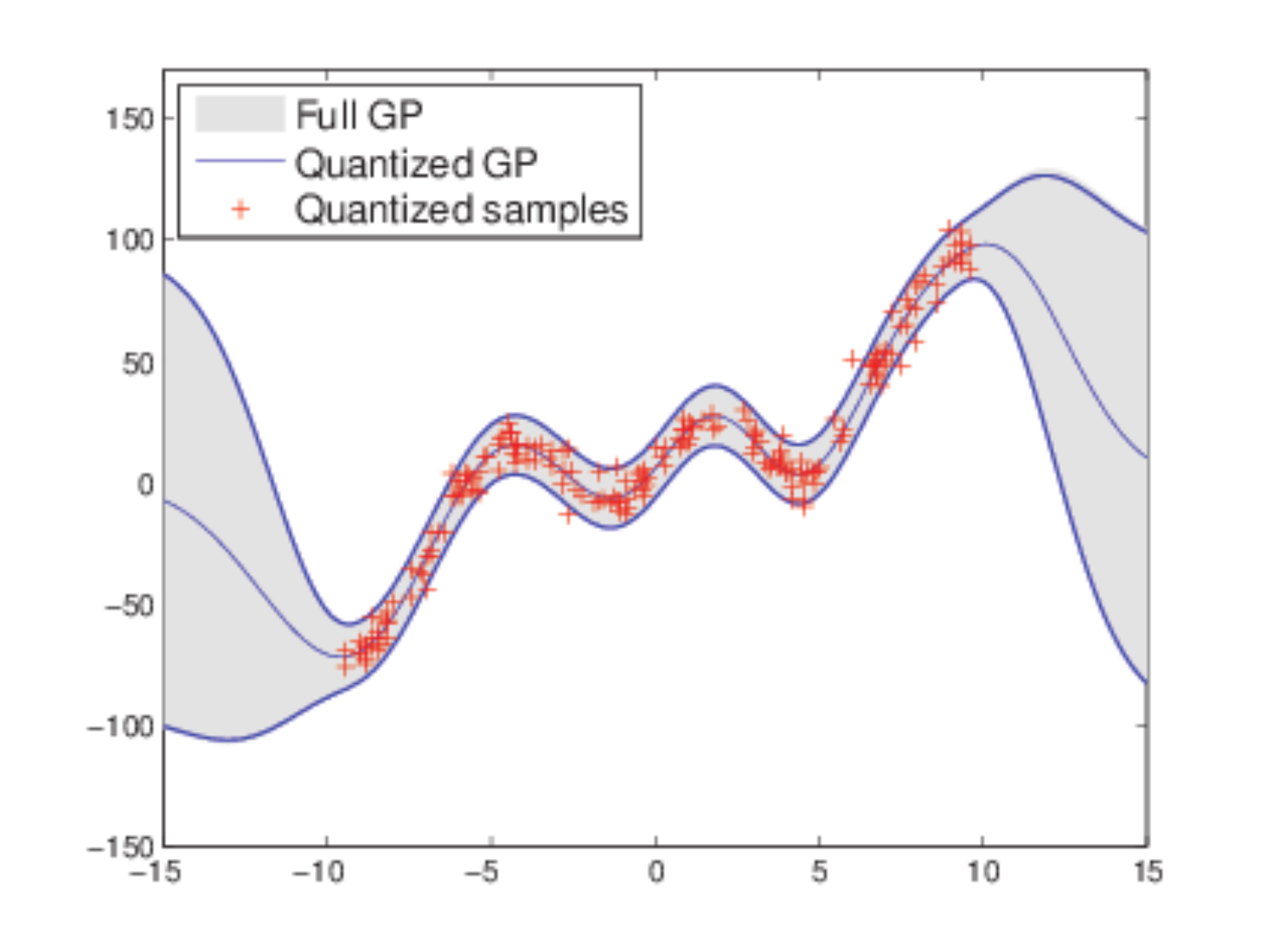}%
		\label{subfig:1d-8bit}}		
	\caption{Training of GP regression on a $1$-dimensional dataset for different bit rates: $R = 1,\cdots, 8$. Quantized inputs are shown by red $+$ markers.}
	\label{fig:1d-gpr}
\end{figure*}

In order to evaluate the trained GPs on quantized data we have used three popular real datasets: the SARCOS ($44,484$ training, $4,449$ test, 21 attributes), the KIN40K ($10,000$ training, $30,000$ test, 8 attributes), and the ABOLONE ($3,133$ training, $1,044$ test, 8 attributes). Learning  several regression models on these datasets has been studied in previous works such as \cite{titsias2009variational}. Since the whole datasets are too large to be trained, 1000 samples selected at random are used for the training GPs. To make a reference, a subset of unquantized data (SD) is used to train GPs at a center machine.  We refer GPs trained using SD method as full GP. The training dataset is randomly distributed across $40$ machines. Inputs are normalized to have  zero mean and unit variance on the training set.  Outputs (target values) are also centered around zero. We compare our method with previous distributed GP models that are based on PoE \cite{ng2014hierarchical} such as BCM \cite{tresp2000bayesian} and rBCM \cite{deisenroth2015distributed}.


\figurename{} \ref{fig:sarcos-abolone-linear-kernel} shows the error of regression for SARCOS and ABOLONE datasets. We have used the linear kernel function in \eqref{eq:linear-kernel} for our GP model.
The regression error is calculated using the standard mean squared error (SMSE) given by $\frac{1}{N} \frac{\Vert \vect{y}_* - \hat{\vect{y}}_* \Vert ^ 2}{var(\vect{y}_*)}$, where $\vect{y}_*$ and $\hat{\vect{y}}_*$ are the target and prediction vectors, respectively. As can be seen from \figurename{} \ref{fig:sarcos linear kernel} for SARCOS dataset the broadcast model crosses the rBCM error line at $R = 16$. Since the input dimension of SARCOS is $21$, by consuming about $0.8$ bit for each dimension, the broadcast model outperforms rBCM. The broadcast model is near the full GP around $40$ bits per sample or about 2 bits per dimension.

\figurename{} \ref{fig:abolone linear kernel} shows similar results on ABOLONE dataset. The broadcast model outperforms rBCM around $R = 16$, which leads to $2$ bits for each dimension. The broadcast model converges to the full GP around $50$ bits (i.e. about $6$ bits per dimension).

The GP model with linear kernel function in \eqref{eq:linear-kernel} does not converge for KIN40K dataset even for the full GP, thus, we could not plot its results.

\begin{figure}[t]
	\centering
	\subfloat[SARCOS]{\includegraphics[width=.5\linewidth]{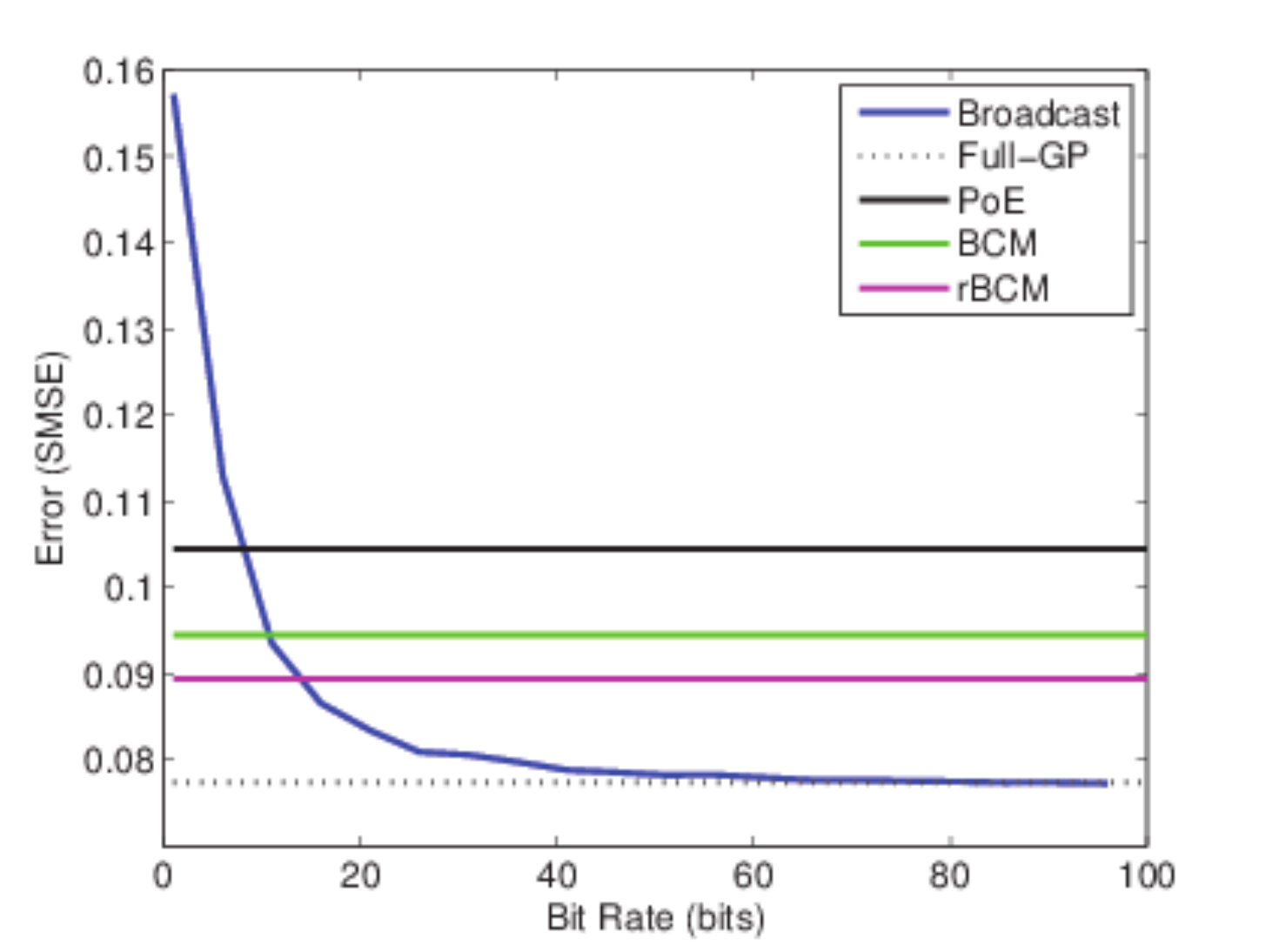} %
	\label{fig:sarcos linear kernel}}
	\subfloat[ABOLONE]{\includegraphics[width=.5\linewidth]{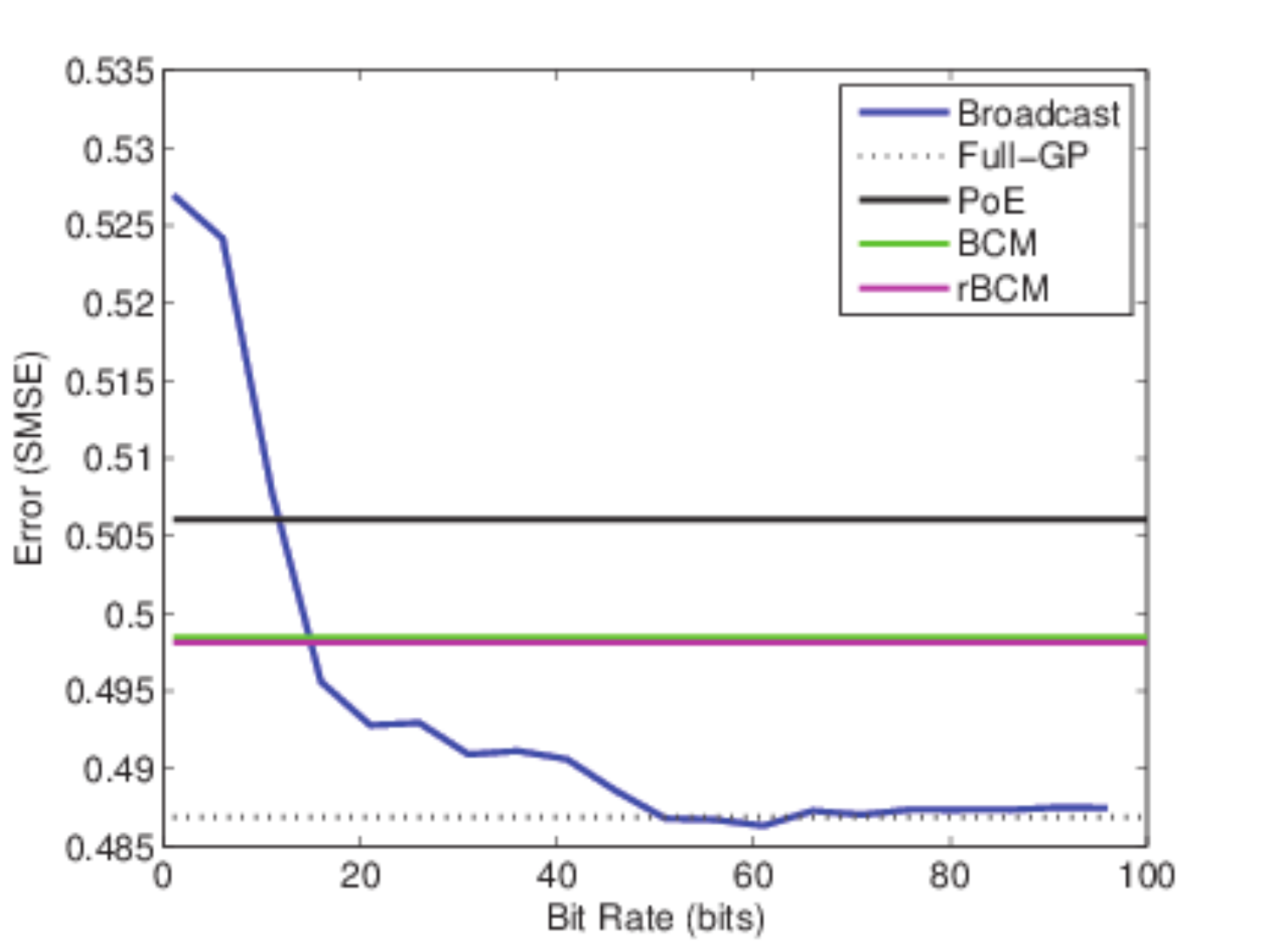} %
	\label{fig:abolone linear kernel}}
	\caption{Error of regression on SARCOS and ABOLONE datasets for GP model with linear kernel function in \eqref{eq:linear-kernel}.}
	\label{fig:sarcos-abolone-linear-kernel}
\end{figure}

\begin{figure*}[t]
	\centering
	\subfloat[SARCOS]{\includegraphics[width=.33\linewidth]{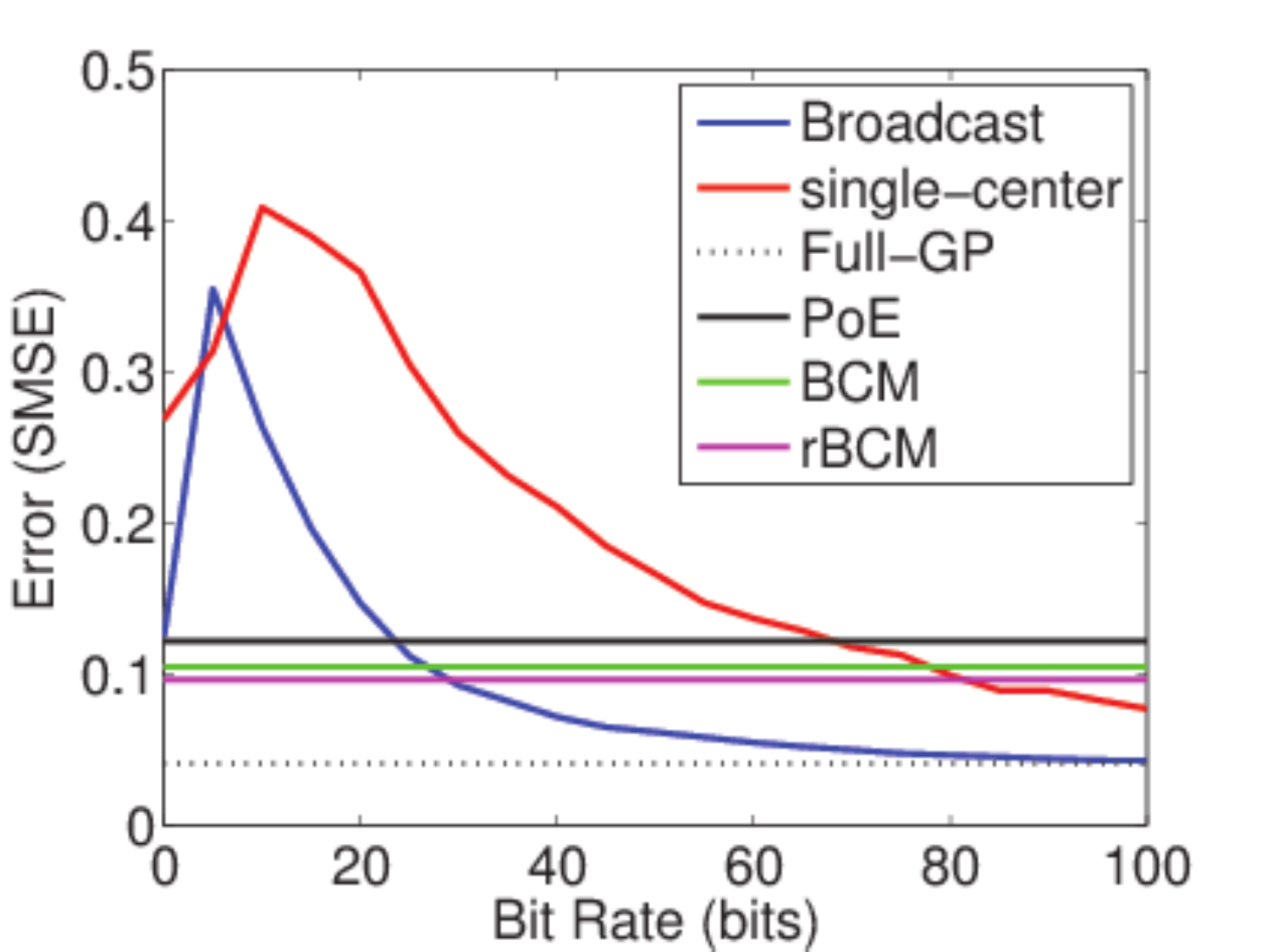} %
	\label{fig:sarcos SE kernel}}
	\subfloat[KIN40K]{\includegraphics[width=.33\linewidth]{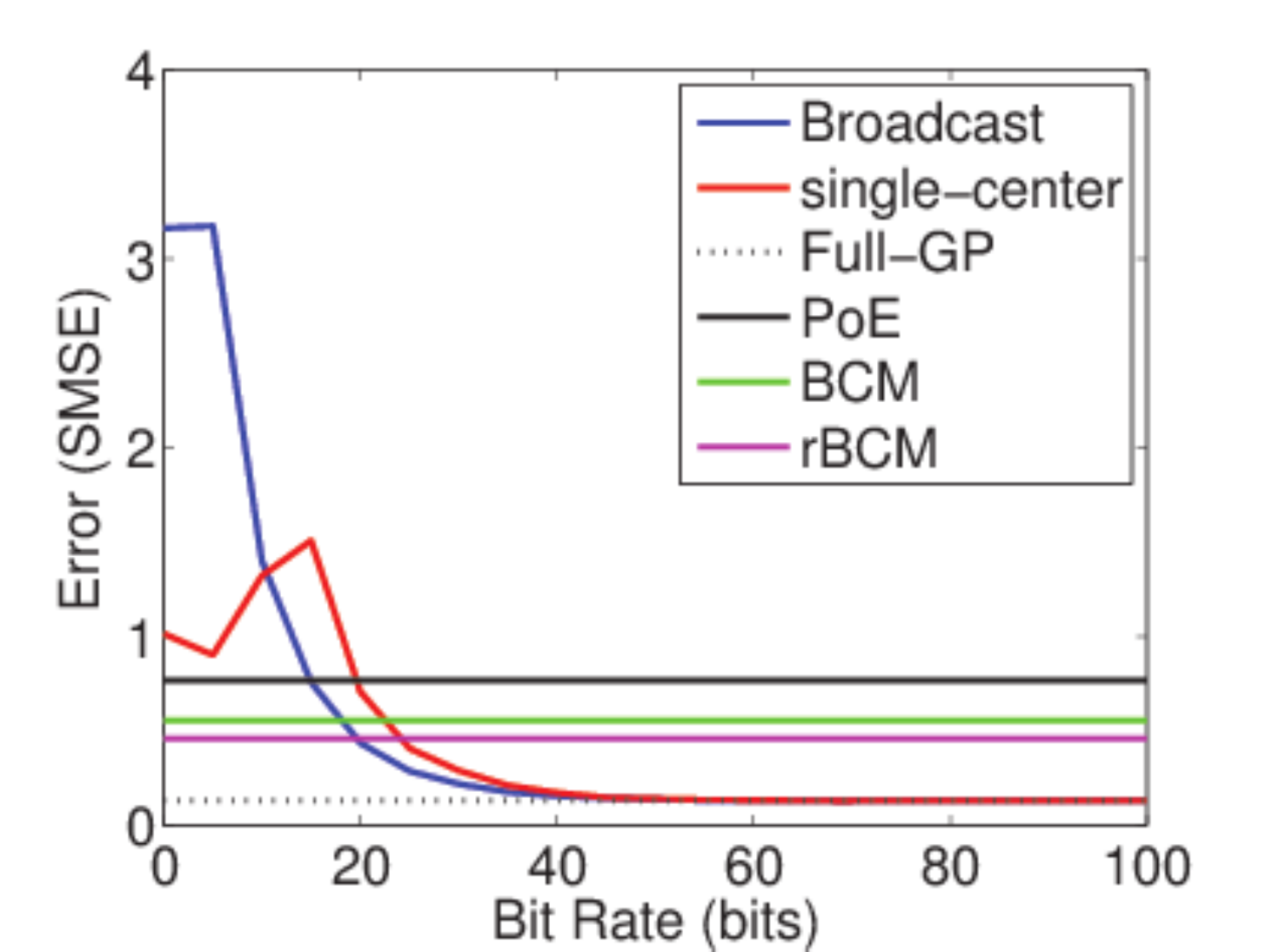} %
	\label{fig:kin40k SE kernel}}
	\subfloat[ABOLONE]{\includegraphics[width=.33\linewidth]{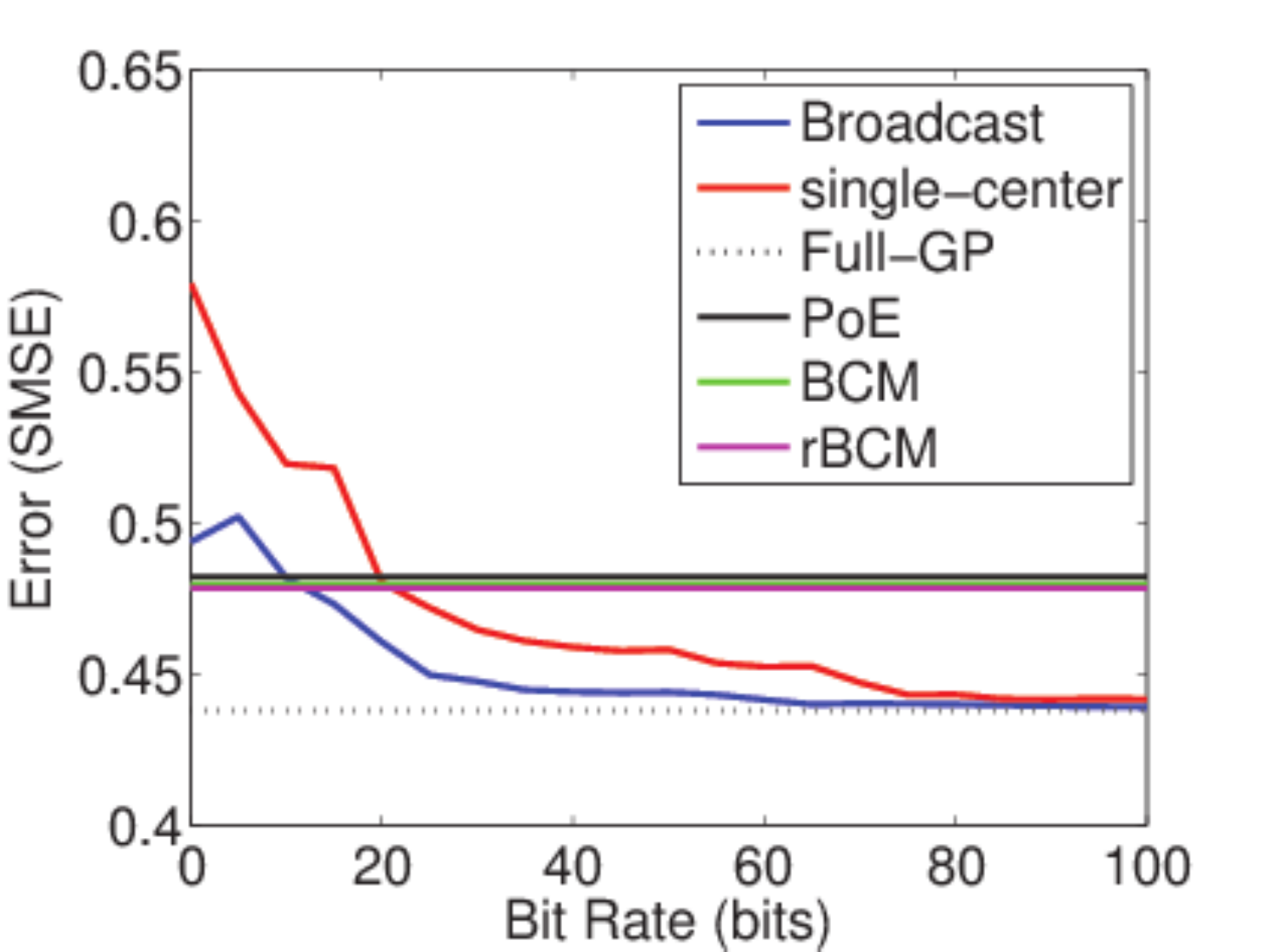} %
	\label{fig:abolone SE kernel}}
	\caption{Error of regression on SARCOS, KIN40K, and ABOLONE datasets for GP with squared exponential kernel function \eqref{eq:squared exp kernel}.}
	\label{fig:sarcos-kin40k-abolone}
\end{figure*}

\figurename{} \ref{fig:sarcos-kin40k-abolone} shows the result of a similar experiment on the three datasets for GP models with squared exponential kernel function:
\begin{equation} \label{eq:squared exp kernel}
k(\vect{x_p}, \vect{x_q}) = \sigma_s e^{-\frac{\Vert \vect{x_p} - \vect{x_q} \Vert^2}{\ell^2}},
\end{equation}
where $\sigma_s$ and $\ell$ are the hyper-parameters of the model. \figurename{} \ref{fig:sarcos SE kernel} shows that our methods outperforms rBCM after spending  about $25$ bits per sample  on SARCOS dataset in the broadcast GP. As mentioned before, the input dimension of SARCOS dataset is $21$. Hence,  each dimension is encode on the average by about $1$ bit.  For single-center GP, the error curve crosses the rBCM error line when about $4$ bits are used per dimension.

\figurename{} \ref{fig:kin40k SE kernel} shows the result of the same experiment for the KIN40K dataset. Here, for bit rate around $20$ we outperform both BCM and rBCM in the broadcast GP. The single-center GP has slightly more error values than the broadcast GP. The input space of the KIN40K dataset is $8$-dimensional, thus, our broadcast GP model requires on the average about $2.5$ bits per dimension to outperform both BCM and rBCM. The larger average bit rate shows that the attributes of KIN40K are more informative than that of SARCOS. The error curves of ABOLONE dataset are shown in \figurename{} \ref{fig:abolone SE kernel}. For the three datasets we see that the zero rate error for the broadcast model is smaller than the error at rate $R=5$. This shows that the highly distorted data could decrease the performance of the model.

The preceding experiments reveal that in case of highly limited communication capacity our scheme performs worse than the rBCM method where none of data points is transferred between the machines.  This suggests that when the communication is highly limiting, it is not a good strategy to transmit large number of data points with low qualities. It is better to transmit lower number of data samples with an acceptable quality. 


To overcome this issue, one idea is to use our quantization method on a subset of samples.
To select effective representatives of samples,  we have used the  method of \cite{titsias2009variational} in which inducing (variational) variables are found.  We, then,  quantize the inducing variables and transmit them to the center machine. In fact, each machine runs locally the method of Titsias \cite{titsias2009variational} to find inducing variables, then applies our proposed quantization methods on them. \figurename{} \ref{fig:sparse-kin40k} shows the error of single-center GP model which adopts this method for distributed GP learning on KIN40K dataset. As can be seen from the figure, the results are significantly better than non-sparse models in \figurename{} \ref{fig:sarcos-kin40k-abolone}.

It is also possible to find a global inducing set by communication between machines using method of \cite{gal2014distributed} which is based on an iterative algorithm. The communication cost of \cite{gal2014distributed} depends on the size of inducing set  as well as the rate of convergence and is relatively high. 

\begin{figure}[t]
	\centering
	\includegraphics[width=.9\linewidth]{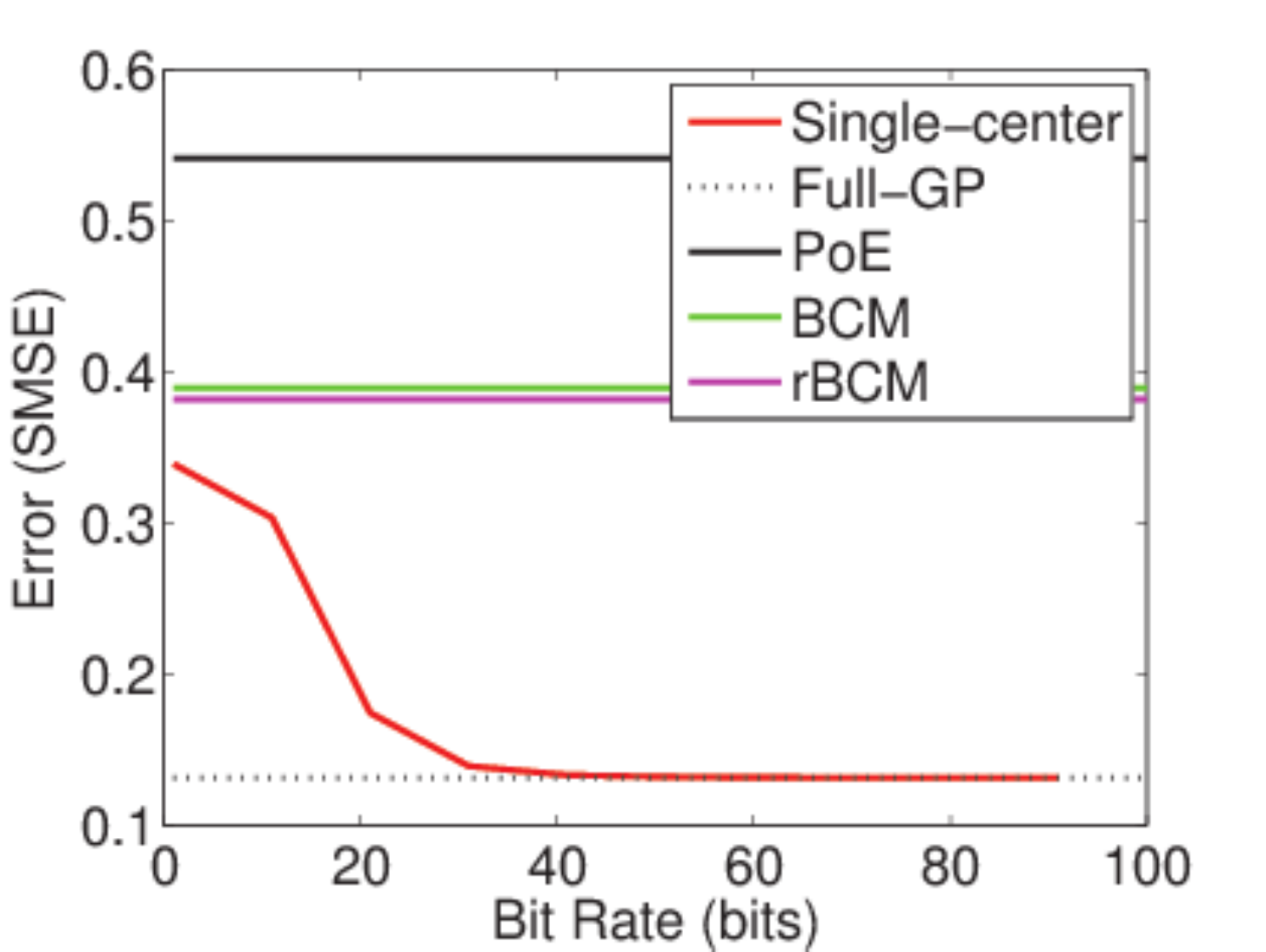}
	\caption{Error of regression for sparse GP with quantized inducing (variational) variables on KIN40K dataset.}
	\label{fig:sparse-kin40k}
\end{figure}

\section{Conclusion}
In this paper, we have proposed three methods for computing inner-products between Gaussian vectors with minimum distortion  in distributed systems. We have shown that the well-known reverse water-filling algorithm on the product of local covariance matrices is the optimal coding scheme for Gaussian variables and it is a lower bound for other distributions. We have also proposed a simple per-symbol coding scheme which operates near the optimal one. The third proposed method is based on a dimension reduction scheme that is similar to the PCA except its objective function is the inner-product distortion. It is also proved that the optimal basis vectors for dimension reduction is obtained from product of local covariance matrices. This is similar to the result of the reverse water-filling scheme where the product of local matrices play the key role in the procedure of encoding.
 
Through several experiments, we have observed that estimation error in GP regressions drops significantly by spending a few bits per each symbol resulting in near the full GP. We also experimentally show that by combining the sparse GP models and the proposed quantization methods, we can obtain a model outperforming PoE-based models even for very low bit rates. This suggests that optimizing inducing points in sparse GPs alongside their quantization will result in a better performance in distributed settings.





\ifCLASSOPTIONcaptionsoff
  \newpage
\fi



\bibliographystyle{IEEEtran}
\bibliography{IEEEabrv,./refs}

\begin{thebibliography}{10}
\providecommand{\url}[1]{#1}
\csname url@samestyle\endcsname
\providecommand{\newblock}{\relax}
\providecommand{\bibinfo}[2]{#2}
\providecommand{\BIBentrySTDinterwordspacing}{\spaceskip=0pt\relax}
\providecommand{\BIBentryALTinterwordstretchfactor}{4}
\providecommand{\BIBentryALTinterwordspacing}{\spaceskip=\fontdimen2\font plus
\BIBentryALTinterwordstretchfactor\fontdimen3\font minus
  \fontdimen4\font\relax}
\providecommand{\BIBforeignlanguage}[2]{{%
\expandafter\ifx\csname l@#1\endcsname\relax
\typeout{** WARNING: IEEEtran.bst: No hyphenation pattern has been}%
\typeout{** loaded for the language `#1'. Using the pattern for}%
\typeout{** the default language instead.}%
\else
\language=\csname l@#1\endcsname
\fi
#2}}
\providecommand{\BIBdecl}{\relax}
\BIBdecl

\bibitem{shvachko2010hadoop}
K.~Shvachko, H.~Kuang, S.~Radia, and R.~Chansler, ``The hadoop distributed file
  system,'' in \emph{2010 IEEE 26th symposium on mass storage systems and
  technologies (MSST)}.\hskip 1em plus 0.5em minus 0.4em\relax IEEE, 2010, pp.
  1--10.

\bibitem{ghemawat2003google}
S.~Ghemawat, H.~Gobioff, and S.-T. Leung, ``The google file system,'' in
  \emph{ACM SIGOPS operating systems review}, vol.~37, no.~5.\hskip 1em plus
  0.5em minus 0.4em\relax ACM, 2003, pp. 29--43.

\bibitem{dean2008mapreduce}
J.~Dean and S.~Ghemawat, ``Mapreduce: simplified data processing on large
  clusters,'' \emph{Communications of the ACM}, vol.~51, no.~1, pp. 107--113,
  2008.

\bibitem{zaharia2010spark}
M.~Zaharia, M.~Chowdhury, M.~J. Franklin, S.~Shenker, and I.~Stoica, ``Spark:
  cluster computing with working sets.'' \emph{HotCloud}, vol.~10, pp. 10--10,
  2010.

\bibitem{meng2016mllib}
X.~Meng, J.~Bradley, B.~Yuvaz, E.~Sparks, S.~Venkataraman, D.~Liu, J.~Freeman,
  D.~Tsai, M.~Amde, S.~Owen \emph{et~al.}, ``Mllib: Machine learning in apache
  spark,'' \emph{JMLR}, vol.~17, no.~34, pp. 1--7, 2016.

\bibitem{jordan2015machine}
M.~Jordan and T.~Mitchell, ``Machine learning: Trends, perspectives, and
  prospects,'' \emph{Science}, vol. 349, no. 6245, pp. 255--260, 2015.

\bibitem{ahlswede1986hypothesis}
R.~Ahlswede and I.~Csisz{\'a}r, ``Hypothesis testing with communication
  constraints,'' \emph{IEEE transactions on information theory}, vol.~32,
  no.~4, pp. 533--542, 1986.

\bibitem{amari1998statistical}
T.~Han and S.~Amari, ``Statistical inference under multiterminal data
  compression,'' \emph{IEEE Transactions on Information Theory}, vol.~44,
  no.~6, pp. 2300--2324, 1998.

\bibitem{amari1995parameter}
S.~Amari and T.~S. Han, ``Parameter estimation with multiterminal data
  compression,'' \emph{IEEE transactions on Information Theory}, vol.~41,
  no.~6, pp. 1802--1833, 1995.

\bibitem{han1987hypothesis}
T.~Han, ``Hypothesis testing with multiterminal data compression,'' \emph{IEEE
  transactions on information theory}, vol.~33, no.~6, pp. 759--772, 1987.

\bibitem{cover2012elements}
T.~M. Cover and J.~A. Thomas, \emph{Elements of information theory}.\hskip 1em
  plus 0.5em minus 0.4em\relax John Wiley \& Sons, 2012.

\bibitem{rahman2012optimality}
M.~S. Rahman and A.~B. Wagner, ``On the optimality of binning for distributed
  hypothesis testing,'' \emph{IEEE Transactions on Information Theory},
  vol.~58, no.~10, pp. 6282--6303, 2012.

\bibitem{jordan2016communication}
M.~I. Jordan, J.~D. Lee, and Y.~Yang, ``Communication-efficient distributed
  statistical inference,'' \emph{stat}, vol. 1050, p.~3, 2016.

\bibitem{duchi2014optimality}
J.~C. Duchi, M.~I. Jordan, M.~J. Wainwright, and Y.~Zhang, ``Optimality
  guarantees for distributed statistical estimation,'' \emph{arXiv preprint
  arXiv:1405.0782}, 2014.

\bibitem{zhang2013information}
Y.~Zhang, J.~Duchi, M.~I. Jordan, and M.~J. Wainwright, ``Information-theoretic
  lower bounds for distributed statistical estimation with communication
  constraints,'' in \emph{Advances in Neural Information Processing Systems},
  2013, pp. 2328--2336.

\bibitem{liu2014distributed}
Q.~Liu and A.~T. Ihler, ``Distributed estimation, information loss and
  exponential families,'' in \emph{Advances in Neural Information Processing
  Systems}, 2014, pp. 1098--1106.

\bibitem{meng2013distributed}
Z.~Meng, D.~Wei, A.~Wiesel, and A.~O. Hero~III, ``Distributed learning of
  gaussian graphical models via marginal likelihoods.'' in \emph{AISTATS},
  2013, pp. 39--47.

\bibitem{broderick2013streaming}
T.~Broderick, N.~Boyd, A.~Wibisono, A.~C. Wilson, and M.~I. Jordan, ``Streaming
  variational bayes,'' in \emph{Advances in Neural Information Processing
  Systems}, 2013, pp. 1727--1735.

\bibitem{balcan2015communication}
M.-F. Balcan, Y.~Liang, L.~Song, D.~Woodruff, and B.~Xie, ``Communication
  efficient distributed kernel principal component analysis,'' \emph{arXiv
  preprint arXiv:1503.06858}, 2015.

\bibitem{boyd2011distributed}
S.~Boyd, N.~Parikh, E.~Chu, B.~Peleato, and J.~Eckstein, ``Distributed
  optimization and statistical learning via the alternating direction method of
  multipliers,'' \emph{Foundations and Trends{\textregistered} in Machine
  Learning}, vol.~3, no.~1, pp. 1--122, 2011.

\bibitem{zhang2014asynchronous}
R.~Zhang and J.~T. Kwok, ``Asynchronous distributed admm for consensus
  optimization.'' in \emph{ICML}, 2014, pp. 1701--1709.

\bibitem{forero2010consensus}
P.~A. Forero, A.~Cano, and G.~B. Giannakis, ``Consensus-based distributed
  support vector machines,'' \emph{Journal of Machine Learning Research},
  vol.~11, no. May, pp. 1663--1707, 2010.

\bibitem{liu2012distributed}
Q.~Liu and A.~T. Ihler, ``Distributed parameter estimation via
  pseudo-likelihood,'' in \emph{Proceedings of the 29th International
  Conference on Machine Learning (ICML-12)}, 2012, pp. 1487--1494.

\bibitem{deisenroth2015distributed}
M.~P. Deisenroth and J.~W. Ng, ``Distributed gaussian processes,'' in
  \emph{International Conference on Machine Learning (ICML)}, vol.~2, no.~2,
  2015, p.~5.

\bibitem{ng2014hierarchical}
J.~W. Ng and M.~P. Deisenroth, ``Hierarchical mixture-of-experts model for
  large-scale gaussian process regression,'' \emph{arXiv preprint
  arXiv:1412.3078}, 2014.

\bibitem{tresp2000bayesian}
V.~Tresp, ``A bayesian committee machine,'' \emph{Neural Computation}, vol.~12,
  no.~11, pp. 2719--2741, 2000.

\bibitem{titsias2009variational}
M.~K. Titsias, ``Variational learning of inducing variables in sparse gaussian
  processes.'' in \emph{AISTATS}, vol.~12, 2009, pp. 567--574.

\bibitem{quinonero2005unifying}
J.~Qui{\~n}onero-Candela and C.~E. Rasmussen, ``A unifying view of sparse
  approximate gaussian process regression,'' \emph{Journal of Machine Learning
  Research}, vol.~6, no. Dec, pp. 1939--1959, 2005.

\bibitem{gal2014distributed}
Y.~Gal, M.~van~der Wilk, and C.~Rasmussen, ``Distributed variational inference
  in sparse gaussian process regression and latent variable models,'' in
  \emph{Advances in Neural Information Processing Systems}, 2014, pp.
  3257--3265.

\bibitem{hensman2013gaussian}
J.~Hensman, N.~Fusi, and N.~D. Lawrence, ``Gaussian processes for big data,''
  in \emph{Conference on Uncertainty in Artificial Intellegence}, pp. 282--290.

\bibitem{rasmussen2006gaussian}
C.~E. Rasmussen and C.~K.~I. Williams, \emph{Gaussian processes for machine
  learning}.\hskip 1em plus 0.5em minus 0.4em\relax MIT Press, 2006.

\bibitem{titsias2010bayesian}
M.~K. Titsias and N.~D. Lawrence, ``Bayesian gaussian process latent variable
  model.'' in \emph{AISTATS}, 2010, pp. 844--851.

\bibitem{snelson2005sparse}
E.~Snelson and Z.~Ghahramani, ``Sparse gaussian processes using
  pseudo-inputs,'' in \emph{Advances in neural information processing systems},
  2005, pp. 1257--1264.

\end{thebibliography}
%

%

%

\begin{IEEEbiography}[{\includegraphics[width=1in,height=1.25in,clip,keepaspectratio]{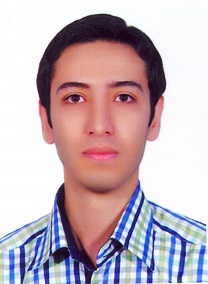}}]{Mostafa Tavassolipour}
received the B.Sc. degree from Shahed University, Tehran, Iran, in 2009, and the M.Sc. degree from Computer Engineering department of Sharif University of Technology (SUT), Tehran, Iran, in 2011. Currently, He is a Ph.D. student of Artificial Intelligence program at Computer Engineering Department of Sharif University of Technology. His research interests include machine learning, image processing, information theory, content based video analysis, and bioinformatics.
\end{IEEEbiography}

\begin{IEEEbiography}[{\includegraphics[width=1in,height=1.25in,clip,keepaspectratio]{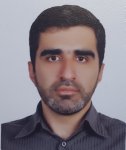}}]{Seyed Abolfazl Motahari}
 is an assistant professor at Computer Engineering Department of Sharif University of Technology (SUT). He received his B.Sc. degree from the Iran University of Science and Technology (IUST), Tehran, in 1999, the M.Sc. degree from Sharif University of Technology, Tehran, in 2001, and the Ph.D. degree from University of Waterloo, Waterloo, Canada, in 2009, all in electrical engineering. From August 2000 to August 2001, he was a Research Scientist with the Advanced Communication Science Research Laboratory, Iran Telecommunication Research Center (ITRC), Tehran. From October 2009 to September 2010, he was a Postdoctoral Fellow with the University of Waterloo, Waterloo. From September 2010 to July 2013, he was a Postdoctoral Fellow with the Department of Electrical Engineering and Computer Sciences, University of California at Berkeley. His research interests include multiuser information theory and Bioinformatics. He received several awards including Natural Science and Engineering Research Council of Canada (NSERC) Post-Doctoral Fellowship.
\end{IEEEbiography}

\begin{IEEEbiography}[{\includegraphics[width=1in,height=1.25in,clip,keepaspectratio]{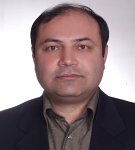}}]{Mohammad Taghi Manzuri Shalmani}
received the B.Sc. and M.Sc.
in electrical engineering from Sharif University of
Technology (SUT), Iran, in 1984 and 1988, respectively.
He received the Ph.D. degree in electrical and
computer engineering from the Vienna University
of Technology, Austria, in 1995. Currently, he is an
associate professor in the Computer Engineering
Department, Sharif University of Technology,
Tehran, Iran. His main research interests include
digital signal processing, stochastic modeling, and
Multi-resolution signal processing.
\end{IEEEbiography}







\end{document}